\documentclass[twoside,11pt]{article}
\usepackage{amsthm}
\usepackage{jmlr2e}

\usepackage{csquotes}
\usepackage{color}
\usepackage{colordvi}
\usepackage{amssymb}

\usepackage{colordvi}
\usepackage{makecell}
\usepackage{float}

\usepackage{graphicx}
\usepackage{subcaption}
\usepackage{enumerate}
\usepackage{wrapfig}
\usepackage{enumitem}
\usepackage{booktabs}

\usepackage{bbm}
\usepackage{bigints}
\usepackage{soul}
\usepackage{pgffor}
\usepackage[colorlinks]{hyperref} 
\usepackage{xcolor}
\usepackage{wrapfig}
\usepackage{minitoc}

\usepackage{bm}

\usepackage[%
    minnames=1,maxnames=99,maxcitenames=1,
    style=numeric,
    sorting=ynt,
    sortcites=true,
    doi=false,url=false,
    giveninits=true,
    hyperref,natbib,backend=biber]{biblatex}
\renewbibmacro{in:}{%
  \ifentrytype{article}{}{\printtext{\bibstring{in}\intitlepunct}}}
\bibliography{sample}

\usepackage[capitalize]{cleveref}

\usepackage{mathtools}

\definecolor{niceblue}{rgb}{0.10, 0.14, 0.76} 

\definecolor{nicered}{rgb}{0.70, 0.0, 0.0} 

\AtBeginDocument{%
\hypersetup{
    citecolor=niceblue,
    linkcolor=red,   
    urlcolor=niceblue,
    linktoc=none}
}

\newcommand{\reals}{\mathbb{R}}

\newcommand{\E}{\mathbb{E}}

\newcommand{\normal}{\mathcal{N}}

\newcommand{\iid}{\textit{iid}~}
\newcommand{\sss}{\mathcal{S}}

\newcommand{\bigO}{\mathcal{O}}

\newtheorem{prop}{Proposition}
\newtheorem{thm}{Theorem}

\newtheorem{lemma}{Lemma}
\newtheorem{corollary}{Corollary}
\newtheorem{definition}{Definition}

\newcommand{\wdst}{$\stackrel{\Huge\longleftrightarrow}{\textrm{\textbf{WD}}}$}
\ShortHeadings{Commutative Width and Depth Scaling}{Hayou}
\firstpageno{1}

\begin{document}
\thispagestyle{plain}
\title{Commutative Width and Depth Scaling\\ in Deep Neural Networks}

\author{\name Soufiane Hayou\thanks{Work partially done at the National University of Singapore.} \email hayou@berkeley.edu \\
       \addr Simons Institute\\
       UC Berkeley}

\editor{}

\maketitle\doparttoc 
\faketableofcontents 
\part{} 

\vspace{-2cm}

\begin{abstract}
This paper is the second in the series \emph{Commutative Scaling of Width and Depth} (\wdst) about commutativity of infinite width and depth limits in deep neural networks. Our aim is to understand the behaviour of neural functions (functions that depend on a neural network model) as width and depth go to infinity (in some sense), and eventually identify settings under which commutativity holds, i.e. the neural function tends to the same limit no matter how width and depth limits are taken. In this paper, we formally introduce and define the commutativity framework, and discuss its implications on neural network design and scaling. We study commutativity for the neural covariance kernel which reflects how network layers separate data. Our findings extend previous results established in \cite{Hayou2023WidthDepth} by showing that taking the width and depth to infinity in a deep neural network with skip connections, when branches are suitably scaled to avoid exploding behaviour, result in the same covariance structure no matter how that limit is taken. This has a number of theoretical and practical implications that we discuss in the paper. The proof techniques in this paper are novel and rely on tools that are more accessible to readers who are not familiar with stochastic calculus (used in the proofs of \wdst(I))).
\end{abstract}

\section{Introduction}

The success of large language and vision models has recently amplified an existing trend of research on large size neural network. There are generally two ways to increase the size of a neural network model: increasing the width, for instance the number of neurons in hidden layers in a fully-connected network, the number of channels in a convolutional network, or the number of attention heads in a transformer architecture; and increasing the depth of the network, i.e. the number of layers. A suitable appraoch to understand the behaviour of large neural networks is by analyzing some pre-defined quantity as the width and/or depth tend to infinity. While the width limit by itself is now relatively well understood in different contexts \citep{neal, samuel2017, lee_gaussian_process,hayou2019impact, yang2021tensor_iv}, the depth limit and the interaction between the two have not been studied as much.
In particular, given some pre-defined quantity of interest that depends on the network model, a basic question is: \emph{do these two limits commute?} (in the sense that the behaviour of the quantity of interest as width and depth go to infinity does not change depending on the order of which these limits are taken). One statistical quantity of interest is the \emph{neural covariance} kernel which reflects how layers in a neural network model separate input data. In this context, recent literature suggests that, at initialization, in certain kinds of multi-layer perceptrons (MLPs) or residual neural networks (ResNets) with scaled main branch, the depth and width limits generally do \emph{not} commute \citep{li2022sde, noci2023shaped}; this would imply that in practice, such networks would behave quite differently depending on whether width is much larger than depth or the other way around. However, in the case of ResNets with suitably scaled residual blocks, recent work \citep{Hayou2023WidthDepth} showed that, to the contrary, at initialization, for a ResNet with blocks scaled the natural way so as to avoid blowing up the output, the width and depth limits \emph{do commute}. An interesting practical implication of this result is that it justifies prior calculations that take the width limit first, then depth, to understand the behavior of deep residual networks, such as prior works in the signal propagation literature \citep{samuel2017, yang2017meanfield, hayou21stable}.

In this work, we introduce and formalize the framework of commutativity of the width and depth limits and generalize (and improve) existing results on the covariance from \citep{Hayou2023WidthDepth}  for arbitrary sequences of scaling factors; these sequences are used to scale the residual blocks so as to avoid exploding behaviour as depth grows. \cref{tab:cs} shows the difference between this work and the previous work in the \wdst~series.
We discuss the theoretical and practical implications of commutativity by addressing the natural question; \emph{why should  we care about commutativity at all?} (see \cref{sec:setup}).

In addition to the significance of the results and the new framework, the mathematical novelty of this paper lies in the proof techniques: in contrast to \citep{Hayou2023WidthDepth} where the depth limit is taken first (fixing the width), followed by the width limit, we first take the width to infinity this time, which is a more conventional approach in the theory of signal propagation in deep networks. As such, the proof techniques in this paper can be seen as `orthogonal' to the machinery developed in \citep{Hayou2023WidthDepth}, and are more accessible to readers who are not familiar with stochastic calculus. Our results provide new insights into the behavior of deep neural networks with general depth scaling factors with implications on the design and analysis of these networks.

\begin{table}
    \centering
    \begin{tabular}{|c|c|c|c|}
    \toprule
        Paper & Block scaling & Neural Functions & Proof Techniques\\
        \midrule
         \wdst(I) (\cite{Hayou2023WidthDepth})&  \thead{$1/\sqrt{depth}$}& \thead{Neural Covariance\\
         Neural Distribution} & \thead{Tools from \\Stochastic Calculus} \\
         \midrule
         \wdst(II) (This work)&  \thead{General \\
         Block Scaling}& \thead{Neural Covariance}& \thead{Standard\\ Concentration results}\\
         \bottomrule
    \end{tabular}
    \caption{Commutative Width and Depth Scaling Series. Block scaling refers to a scaling factor in front of the residual block. Neural functions are formally defined in \cref{sec:setup}.}
    \label{tab:cs}
\end{table}

All the proofs are deferred to the appendix and referenced after each result. Empirical evaluations are provided to illustrate the theoretical results.

\section{Related Work}
The theoretical analysis of randomly initialized neural networks with an infinite number of parameters has yielded a stream of interesting results, both theoretical and practical. A majority of this research has concentrated on examining the scenario in which the width of the network is taken to infinity while the depth is considered fixed. However, in recent years, there has been a growing interest in exploring the large depth limit of these networks. In this overview, we present a summary of existing results on this topic, though it is not exhaustive. A more comprehensive literature review is provided in \cref{sec:comprehensive_lit_review}.\\
\noindent$\bullet~$\textit{Infinite-Width Limit:}
The study of the infinite-width limit of neural network architectures has been a topic of significant research interest, yielding various theoretical and algorithmic innovations. These include initialization methods, such as the Edge of Chaos \citep{poole, samuel2017, yang2017meanfield, hayou2019impact}, and the selection of activation functions \citep{hayou2019impact, martens2021rapid, zhang2022deep, wolinski2022gaussian}, which have been shown to have practical benefits. In the context of Bayesian analysis, the infinite-width limit presents an intriguing framework for Bayesian deep learning, as it is characterized by a Gaussian process prior. Several studies (e.g. \cite{neal, lee_gaussian_process, yang_tensor3_2020, matthews, hron20attention}) have investigated the weak limit of neural networks as the width goes to infinity, and have demonstrated that the network's output converges to a distribution modeled by a Gaussian process. Bayesian inference utilizing this ``neural" Gaussian process has been explored in \citep{lee_gaussian_process, hayou21stable}. \footnote{It is worth mentioning that kernel methods such as NNGP and NTK significantly underperform properly tuned finite-width network trained using SGD, see \cite{yang2022efficient}.}
    
\noindent$\bullet~$\noindent\textit{Infinite-Depth Limit:} The infinite-depth limit of randomly initialized neural networks is a less explored topic compared to the infinite-width limit. Existing results can be categorized depending on how the two limits are taken. For instance, in the case of sequential limits, the width of the neural network is taken to infinity first, followed by the depth. This limit has been extensively utilized to explore various aspects of neural networks, such as examining the neural covariance, deriving the Edge of Chaos initialization scheme (\citep{samuel2017, poole, yang2017meanfield, hayou2019impact}), evaluating the impact of the activation function \citep{hayou2019impact, martens2021rapid, zhang2022deep}, studying the behavior of the Neural Tangent Kernel (NTK) \citep{hayou_ntk, xiao2020disentangling}, and deriving the distribution of the limiting networks at initialization \citep{hayou21stable,cirone2023neural}. Another interesting limit is the proportional limit where the ratio of depth to width is fixed, and both are jointly taken to infinity. In \citep{li21loggaussian}, the authors showed that for a particular type of residual neural networks (ResNets), the network output exhibits a (scaled) log-normal behavior in this limit, which differs from the sequential limit in which the width is first taken to infinity followed by depth, in which case the distribution of the network output is asymptotically normal (\citep{samuel2017, hayou2019impact}). Additionally, in \citep{li2022sde}, the authors examined the neural covariance of a multi-layer perceptron (MLP) in the joint limit and proved that it weakly converges to the solution of a Stochastic Differential Equation (SDE). Other works have investigated this limit and found similar results \citep{noci2021precise, zavatone2021exact, Hanin2019product, hanin2022correlation, noci2023shaped}. A third interesting approach is the general limit $\min\{n,L\}\to \infty$, where width and depth can to infinity in any order. To the best of our knowledge, this limit was only studied in \citep{Hayou2023WidthDepth} (\wdst(I)) where convergence of the neural covariance in this limit was established for suitably scaled ResNet, implying that the infinite width and depth limits \emph{commute}.\\
\noindent$\bullet~$ \emph{Commutativity of the limits:} given some neural function (a function that depends on the network parameters, to be defined later), we can think about whether taking the width and depth limits result in different behaviour depending on how this limit is taken. In this context, we distinguish between two notions of commutativity: \emph{weak commutativity} which implies that the sequential limits ``width $\to \infty$, then depth$\to\infty$'' and ``depth $\to \infty$, then width$\to\infty$'' yield the same limit, and \emph{strong commutativity} which implies that the limit ``$\min\{${width, depth$\}\to\infty$'' exists and is unique. Moreover, limits are always defined in some sense (e.g. $L_2$, weak limit etc.). In this context, strong commutativity was shown in \cite{Hayou2023WidthDepth} for neural distribution (distribution of a neuron in the network) with Wasserstein distance and for neural ccvariance kernel with $L_2$ distance. In a recent work by \cite{cirone2023neural}, the authors showed weak commutativity of the neural distribution for \emph{controlled} ResNets, a form of ResNets with scaling factors given by the increments from some reference function.\footnote{In \cite{cirone2023neural}, while the results are with weak commutativity, the proofs can be in-principle extended to show strong commutativity for the weak convergence of the neural distribution.} In this work, we will focus on the neural covariance instead of neural distribution and show strong commutativity under general depth scaling.

\section{Setup and Definitions: Commutativity and Neural Functions} \label{sec:setup}
When analyzing the asymptotic behavior of randomly initialized neural networks, various notions of probabilistic convergence can be employed, depending on the context. In this work, we particularly focus on strong convergence, defined to be the $L_2$ convergence as described in the following definition.
\begin{definition}[Strong Convergence]
Let $d \geq 1$. We say that a sequence of  $\reals^d$-valued random variables $(X_k)_{ k\geq 1}$ converges in $L_2$ (or strongly) to a continuous random variable $Z$ if $\lim_{k \to \infty} \|X_k - Z\|_{L_2} =0$, where the $L_2$ is defined by $\|X\|_{L_2} = \left(\E[\|X\|^2] \right)^{1/2}$.
\end{definition}

With this notion of strong convergence, we are now ready to introduce the commutativity framework for general neural network models.
\paragraph{Notation.} Throughout the paper, the width and depth of a neural network model are denoted by $n$ and $L$, respectively, and the input dimension is denoted by $d$. We write $[N]:=\{1,2,\dots,N\}$ for any $N\geq1$.\\

Let us now consider a general neural network model of width $n\geq 1$ and depth $L\geq 1$, given by 
\begin{equation}\label{eq:general_network}
\begin{cases}
Y_0(a) = W_{in} a, \quad a \in \reals^d\\
Y_l(a) = \mathcal{F}_l(W_l, Y_{l-1}(a)), \hspace{0.1cm} l \in [L], Y_l(a) \in \reals^n,\\
\end{cases}
\end{equation}
where $\mathcal{F}_l$ is a mapping that defines the nature of the $l^{th}$ layer and $W_{in} \in \reals^{n\times d}, W_l \in \reals^{n\times n}$ are model weights. For the sake of simplification, we omit the dependence of $Y_l$ on $n$ and $L$ in the notation. We refer to the vectors $\{Y_l, l=0, \dots, L\}$ as \emph{pre-activations}. Let $\theta_{n,L}=(W_{in},W_1, \dots, W_l)$ be the model weights and assume that $\theta_{n,L}^0 \sim \mu_{n,L}^0$, where $\theta_{n,L}^0$ are the weights at initialization  and $\mu^0$ is a distribution that (naturally) depends on network width $n$ and depth $L$. Let us now define the notion of neural functions.

\begin{definition}[Neural Function]
Given a general neural network model (\cref{eq:general_network}) of width $n$ and depth $L$, a set of network inputs $\mathbf{a} = (a_1,a_2,\dots,a_k )\in (\reals^d)^k$, a neural function $T$ is any function of the form $T(n,L, \mathbf{a}) = \mathcal{G}(\theta_{n,L}^0, \mathbf{a})$, where $\mathcal{G}$ is a general mapping with output in $\reals$.\footnote{This definition of neural functions can be extended to general mappings $\mathcal{G}$ with outputs in $\reals^p$ for some $p\geq 1$. This is not required in this paper since we will be focusing on neural covariance kernel which has output in $\reals$.}
\end{definition}

Note that (almost) any quantity of interest in the training process of neural networks can be represented as a neural function. This remark was first observed in the series of Tensor Programs \citep{yang2021tensor_iv} where the result of any neural computation can be seen as a random quantity where the randomness is inherited from the initialization weights. The training dataset is considered deterministic in this case and consists of a sequence of inputs $(a_1,a_2, \dots, a_k)$. In this paper, we particularly think of neural functions as \emph{proxy} functions that track some behaviour of the network as we scale width and depth with the goal of providing insights on scaling strategies (see below for a specific choice of the neural function). With this definition of neural functions, we now formalize the notion of commutativity of the width and depth limits.
\begin{definition}[Commutativity]\label{def:commutativity}
Given a neural function $T$,\footnote{Note that by definition, a neural function is associated with a network model. When we consider a neural function $T$, the underlying model is assumed to be fixed.} we say that $T$ satisfies universality for the width and depth limits if for any set of inputs $\mathbf{a} = (a_1,a_2,\dots, a_k)$, $T(n,L,\mathbf{a})$ converges in $L_2$ in the limit $\min{\{n,L\}}\to\infty$. 
\end{definition}

We can define a weak notion of commutativity where only sequential limits are considered, i.e. $n$ or $L$ limits are taken in a sequential order. 
\begin{definition}[Weak Commutativity]
Given a sequence of neural functions $T=(T_{n,L})_{n,L\geq 1}$, we say that $T_{n,L}$ satisfies commutativity for the width and depth limits if for any set of inputs $(a_1,a_2,\dots, a_k)$, both $\lim\limits_{L\to\infty}\lim\limits_{n\to \infty}T_{n,L}(a_1,a_2,\dots, a_k)$ and $\lim\limits_{n\to\infty}\lim\limits_{L\to \infty}T_{n,L}(a_1,a_2,\dots, a_k)$ exist in $L_2$ and are equal.
\end{definition}

Weak commutativity is trivially implied by commutativity. Intuitively, weak commutativity only deals with the `extreme' scenarios $L \gg n \gg 1$ and $n \gg L \gg 1$ and does not consider the cases where for instance $L \approx n \gg 1$.

\paragraph{Implications of Commutativity.} Naturally, one might ask why we should care about commutativity at first. Commutativity of width and depth limits in neural networks holds significant importance for several compelling reasons:
\begin{enumerate}
    \item \emph{Unification of Width and Depth Scaling:} when we aim to scale a neural network for improved performance, we often encounter scenarios where we must decide whether to increase the network's width or depth. Each of these choices generally lead to different design considerations, including variations in initialization schemes, activation functions, and learning rates. However, commutativity of the width and depth limits for some neural function $T$ ensures that regardless of how we scale the network—whether by increasing width before depth, growing both width and depth proportionally, or taking width to infinity before depth—the resulting limiting behavior remains consistent. This means that once an effective scaling strategy is identified for a specific scenario with large width and depth, it remains a viable choice as long as both width and depth are large, simplifying the scaling process.
    \item \emph{Robust Scaling:} as a result of commutativity, scaling the width and depth becomes robust to extreme changes in neural functions. This allows some flexibility in the scaling procedure; in practice, one might want to increase width significantly while fixing depth, or the opposite, while preserving desirable properties captured by the neural function. 
    \item \emph{Transfer of Insights:} commutativity facilitates the transfer of insights from simplified theoretical settings to practical applications. When dealing with neural networks of large width and depth, it can be challenging to analyze their behavior directly. However, commutativity allows us to explore different limits, such as taking width to infinity first and then depth or vice versa, to gain a better understanding of the network's behaviour. Because the limit is universal (no matter how width and depth go to infinity), the insights we get in the simplified setting (e.g. sequential limit) transfers to all settings (e.g. when depth is of the same order as width).

    \item \emph{Commutativity is Feasible in Practice: }we show that by introducing a simple scaling factor in front of the residual block in ResNets, commutativity holds for the neural covariance function at initialization (defined below). This neural function is used as a measure of how network layers separate input data, and led to many interesting practical methods (initialization schemes, neural network Gaussian process, choice of the activation function etc.) \citep{lee_gaussian_process, samuel2017, hayou2019impact}. An in-depth discussion on this topic is provided below.
\end{enumerate}

\paragraph{Neural Covariance.} In this paper, we focus on neural functions given by the covariance/correlation functions \emph{at initialization}. Given two inputs $a,b\in \reals^d \backslash \{0\}$,\footnote{Here, we assume that the inputs are non-zero, other all the pre-activations $Y_l$ are zero, and the correlation is undefined in this case. All the results in this paper are trivial if $a=0$ or $b=0$. We will therefore always assume that $a,b\neq 0$.} the neural covariance and correlation kernels at layer $l$ are given by 
\begin{equation*}
\begin{cases}
q_{l,n}(a,b) = \frac{\langle Y_l(a), Y_l(b) \rangle }{n}\\
c_{l,n}(a,b) = \frac{\langle Y_l(a), Y_l(b) \rangle }{\|Y_l(a)\|\|Y_l(b)\|},
\end{cases}
\end{equation*}
where the correlation is only defined when $\|Y_l(a)\|, \|Y_l(b)\| \neq 0$.\\

Note that in general, if commutativity holds for the covariance kernel, then it holds for the neural correlation kernel, and vice-versa. This is true as long as pre-activations norms $\|Y_l(a)\|$ are non-zero with high probability, which is generally satisfied, see \cref{lem:hl} for a rigorous proof of this result. Hereafter, we will interchangeably discuss commutativity for neural covariance and correlation, while stating the theoretical results only for neural covariance. The results on the convergence of neural covariance are stated for two inputs $a, b$, but they can be readily generalized to the case of multiple inputs $a_1,a_2,\dots, a_k \in \reals^d$, where we can define the neural covariance matrix at layer $l$ by 
\begin{equation*}
\mathbf{q}_{l,n}(a_1,a_2, \dots, a_k) =
\begin{pmatrix}
 q_{l,n}(a_1,a_1) & \dots & q_{l,n}(a_1,a_k)\\
 \vdots & \ddots & \vdots\\
q_{l,n}(a_k,a_1) & \dots & q_{l,n}(a_k,a_k)
\end{pmatrix}.
\end{equation*}

\paragraph{Why neural Covariance/Correlation?}
In the literature on signal propagation, there is a significant interest in understanding the covariance/correlation between the pre-activation vectors $Y_{\lfloor t L \rfloor}(a)$ and $Y_{\lfloor t L \rfloor}(b)$ for two different inputs $a, b \in \reals^d$. A natural question in this context is: \emph{Why should we care about this covariance function?}

It is well-established that even with properly initialized multi-layer perceptrons (MLPs), the network outputs $Y_L(a)$ and $Y_L(b)$ become perfectly correlated (correlation=1) in the limit of ``$n \to \infty$, \emph{then} $L \to \infty$'' \citep{samuel2017, poole, hayou2019impact, yang2019fine}.  This can lead to unstable behavior of the gradients and make the model untrainable as the depth increases and also results in the inputs being non-separable by the network\footnote{To see this, assume that the inputs are normalized. In this case, the correlation between the pre-activations of the last layer for two different inputs converges to 1. This implies that as the depth grows, the network output becomes similar for all inputs, and the network no longer separates the data. This is problematic for the first step of gradient descent as it implies that the information from the data is (almost) unused in the first gradient update.}. To address this issue, several techniques involving targeted modifications of the activation function have been proposed \citep{martens2021rapid, zhang2022deep}. In the case of ResNets, the correlation still converges to 1, but at a polynomial rate \citep{yang2017meanfield}. A solution to this problem has been proposed by introducing well-chosen scaling factors in the residual branches, preventing the correlation kernel from converging to 1. This analysis was carried in the limit ``$n \to \infty$, then, $L \to \infty$'' in \citep{hayou21stable}, and recently extended in \cite{Hayou2023WidthDepth} to the case where ``$\min(n,L)\to \infty$'', showing that commutativity holds in this case. Some of these works have provided empirical evidence showing an association between favorable characteristics of the neural covariance/correlation and good trainability properties of deep networks.\footnote{By favorable characteristics of the neural covariance, we refer for instance to non-degeneracy as $L \to \infty$ as reported in \citep{hayou21stable}.} 

\section{Overview of Existing Results }\label{sec:existing_results}
In this section, we present corollaries of existing results showing different scenarios where commutativity is satisfied or not for the neural covariance. The aim of this section is show that commutativity depends on the neural architecture.

\subsection{Non-Commutativity in MLPs}
Let $d, n, L \geq 1$, and consider a simple MLP architecture given by the following: 
\begin{equation}\label{eq:mlp}
\begin{cases}
Y_0(a) = W_{in} a, \quad a \in \reals^d\\
Y_l(a) = W_l \phi(Y_{l-1}(a)), \hspace{0.1cm} l \in [L],
\end{cases}
\end{equation}
where $\phi: \reals \to \reals$ is the ReLU activation function, $W_{in} \in \reals^{n \times d}$, and $W_l \in \reals^{n \times n}$ is the weight matrix in the $l^{th}$ layer. We assume that the weights are randomly initialized with \iid Gaussian variables $W_l^{ij} \sim \normal(0, \frac{2}{n})$,\footnote{This is the standard He initialization which coincides with the Edge of Chaos initialization \citep{samuel2017}. This is the only choice of the variance that guarantees stability in both the large-width and the large-depth limits.} $W_{in}^{ij} \sim \normal(0, \frac{1}{d})$. While the activation function is only defined for real numbers ($1$-dimensional), we abuse the notation and write $\phi(z) = (\phi(z^1), \dots, \phi(z^k))$ for any $k$-dimensional vector $z = (z^1, \dots, z^k) \in \reals^k$ for any $k \geq 1$.  We refer to the vectors $\{\phi(Y_l), l=0, \dots, L\}$ as \emph{post-activations}.

In the case of the joint limit $n,L \to \infty$ with $n/L$ fixed, it has been shown that the covariance/correlation between $Y_{\lfloor t L \rfloor}(a)$ and $Y_{\lfloor t L \rfloor}(b)$ becomes similar to that of a Markov chain that incorporates random terms. However, the correlation still converges to $1$ in this limit.
\begin{prop}[Correlation, \citep{hayou2019impact, li2022sde}]\label{prop:covariance_mlp}
Consider the MLP architecture given by \cref{eq:mlp} and let $a, b \in \reals^d$ such that $a, b \neq 0$. Then, in the limit ``$n\to \infty$, then $L \to \infty$'' or the the joint limit ``$n ,L \to \infty$, $L/n$ fixed'', the correlation $\frac{\langle Y_{L}(a), Y_{L}(b) \rangle}{ \|Y_{L}(a)\| \|Y_{L}(b)\|}$ converges\footnote{Note that weak convergence to a constant implies also convergence in probability.} weakly to 1.
\end{prop}

The convergence of the correlation to 1 in the infinite depth limit of a neural network poses a significant issue, as it indicates that the network loses all of the covariance structure from the inputs as the depth increases. This results in 
 degenerate gradients (see e.g. \citep{samuel2017}), rendering the network untrainable. To address this problem in MLPs, various studies have proposed the use of depth-dependent shaped ReLU activations, which prevent the correlation from converging to 1 and exhibit stochastic differential equation (SDE) behavior. As a result, the correlation of the last layer does not converge to a deterministic value in this case.
\begin{prop}[Correlation SDE, Corollary of Thm 3.2 in \cite{li2022sde}]\label{prop:covariance_shaped_mlp}
Consider the MLP architecture given by \cref{eq:mlp} with the following activation function $\phi_L(z) = z + \frac{1}{\sqrt{L}} \phi(z) $ (a modified ReLU). Let $a, b \in \reals^d$ such that $a, b \neq 0$. Then, in the joint limit ``$n ,L \to \infty$, $L/n$ fixed'', the correlation $\frac{\langle Y_{L}(a), Y_{L}(b) \rangle}{ \|Y_L(a)\| \|Y_L(b)\|}$  converges weakly to a nondeterministic random variable.\footnote{In \cite{li2022sde}, the authors show that the correlation of $\frac{\langle \phi_L(Y_{L}(a)), \phi_L(Y_{L}(b)) \rangle}{\sqrt{ \|\phi_L(Y_{L}(a))\|} \sqrt{ \|\phi_L(Y_{L}(b))\|}}$ converges to a random variable in the joint limit. Since $\phi_L$ converges to the identity function in this limit, simple calculations show that the correlation between the pre-activations $\frac{\langle Y_{L}(a), Y_{L}(b) \rangle}{\|Y_{L}(a)\| \|Y_{L}(b)\|}$ is also random in this limit.}
\end{prop}
The joint limit, therefore, yields non-deterministic behaviour of the covariance structure. It is easy to check that even with shaped ReLU as in \cref{prop:covariance_shaped_mlp}, taking the width to infinity first, then depth, the result is a deterministic covariance structure. The main takeaway from this section is the following:

\begin{corollary}
With MLPs (\cref{eq:mlp}), the width and depth limits do not commute for the neural covariance/correlation.
\end{corollary}

\subsection{Commutativity with Scaled Residual Networks}
Using the same notation as in the MLP case, consider the following ResNet architecture of width $n$ and depth $L$
\begin{equation}\label{eq:resnet_uniform}
\begin{aligned}
Y_0(a) &= W_{in} a, \quad a \in \reals^d\\
Y_l(a) &= Y_{l-1}(a) + \frac{1}{\sqrt{L}} W_l \phi(Y_{l-1}(a)), \hspace{0.1cm} l \in [1:L],
\end{aligned}
\end{equation}
where $\phi: \reals \to \reals$ is the ReLU activation function. Assume that the weights are randomly initialized with \iid Gaussian variables $W_l^{ij} \sim \normal(0, \frac{1}{n})$, $W_{in}^{ij} \sim \normal(0, \frac{1}{d})$. If we consider the set of scaling factors of the form $L^{-\gamma}$ for $\gamma >0$, then the choice of $\gamma = 1/2$ is the smallest value of $\gamma$ such that the network output do not explode in the infinite-depth limit (see \cref{lemma:variance}). Therefore, in some sense, this scaling is the `optimal' amongst uniform scalings (meaning all residual branches are scaled with the same factor) for two reasons: it stabilizes the network as depth increases, and it does not result in trivial behaviour (see discussion after \cref{prop:thm_wdi}).

With the ResNet architecture \cref{eq:resnet_uniform}, we have the following result for the covariance kernel, which establishes commutativity in this case.

\begin{prop}[Thm 2 in \citep{Hayou2023WidthDepth}]\label{prop:thm_wdi}
Let $a, b \in \reals^d$ such that $a, b \neq 0$ and $a \neq b$. Then, we have the following 
$$
\sup_{t \in [0,1]} \left\| q_{\lfloor t L\rfloor,n}(a,b) - q_t(a,b)\right\|_{L_2} \leq C \left(\frac{1}{\sqrt{n}} + \frac{1}{\sqrt{L}} \right)
$$
where $C$ is a constant that depends only on $\|a\|$, $\|b\|$, and $d$, and $q_t(a,b)$ is the solution of the following differential flow
\begin{equation}
\begin{aligned}
\begin{cases}
\frac{d q_t(a,b)}{dt} &= \frac{1}{2} \frac{f(c_t(a,b))}{c_t(a,b)} q_t(a,b),\\
c_t(a,b) &= \frac{q_t(a,b)}{ \sqrt{q_t(a,a)} \sqrt{q_t(b,b)}},\\
q_0(a,b) &= \frac{\langle a, b \rangle}{d},
\end{cases}
\end{aligned}
\end{equation}
where the function $f: [-1,1] \to [-1,1]$ is given by 
$$
f(z) = \frac{1}{\pi} ( z \arcsin(z) + \sqrt{1 - z^2}) + \frac{1}{2}z.
$$   
\end{prop}

This result suggests that commutativity for the neural covariance depends on the architecture, and holds in this particular case. More importantly, with this residual architecture, taking the width and depth limits to infinity yield a non-trivial limit of the neural covariance given by the function $q_t$. In \citep{hayou21stable}, it was shown that $q_t$ is a universal kernel, meaning that, it is not only non-trivial, but one can approximate any sufficiently smooth function on some compact set with features from the kernel $q_t$. This has a number of implications, especially in the context of neural network Gaussian processes. We invite the reader to check \citep{hayou21stable} for a more in-depth discussion. 
Another recent result showed that trivial behaviour can be avoided by scaling the main branch of the ResNet. The neural covariance converges weakly to a random variable in the proportional limit, which implies that such scaling breaks commutativity.

\begin{prop}[Corrollary of Thm in \citep{noci2023shaped}]\label{prop:shaped_resnet}
Conider a ResNet where the hidden layers are of the form $Y_l(a) = \beta Y_{l-1}(a) + \sqrt{1-\beta^2} W_l \phi_L(Y_{l-1}(a))$, where $\beta \in (0,1)$ is a constant, and $\phi_L$ is the shaped ReLU (defined in \cref{prop:covariance_shaped_mlp}). Then, the width and depth limits for the covariance kernel do not commute in this case.
\end{prop}
Scaling the main branch of the residual network results in a similar behaviour to the MLP case. Intuitively, with the factor $\beta$, the direct contribution of any layer to the main branch decreases exponentially with depth, hence simulating the `multiplicative' nature of MLPs. Note that the use of shaped ReLU is essential with this scaling in order to avoid degeneracy problems; with ReLU, the correlation converges to $1$ in the proportional limit. In the same paper, the authors show a similar result for Transformers which is a more modern residual architecture. \\

With the background information provided above, we are now able to present our findings. In the next section, we demonstrates commutativity of the width and depth limits for a general class of ResNet architectures, extending the results of \citep{Hayou2023WidthDepth}.
\section{Main Results: Commutativity under General Scaling}
In this section, we present our main results regarding commutativity of the width and depth limits under general scaling rules. All the proofs are deferred to the Appendix. We first define the  \emph{sequence
of scaling factors}, a notion that will be frequently used in the paper. 
\begin{definition}[Sequence of Scaling Factors]
A sequence of scaling factors is an infinite triangular array of non-negative real
numbers. It has the form $\alpha=(\alpha_{l,L})_{l\in\{1,\dots,L\},L\geq1}$.
\end{definition}
Visually, one can think of $\alpha$ as an infinite object of the form
$$
\alpha = \begin{cases}
    \alpha_{1,1} \\
    \alpha_{1,2} \quad \alpha_{2,2} \\
    \quad\vdots \quad\quad \vdots \hspace{0.12cm} \ddots\\
    \alpha_{1,L} \quad \dots \quad \alpha_{L,L} \\
    \quad \vdots \quad \dots \quad \dots \hspace{0.2cm} \ddots
\end{cases}
$$
The use of such notation will come handy when we scale up the depth of a neural network. Such sequences will be used to define a scaling strategy as network depth grows.

\paragraph{Setup.} Recall the previously introduced notation, the width and depth of the network are denoted by $n$ and $L$, respectively, and the input dimension is denoted by $d$. Let $n,L,d\geq1$,
and consider the following neural network model with skip connections
\begin{equation}\label{eq:main_resnet}
\begin{cases}
Y_{0}(a)=W_{in}a,\quad a\in\reals^{d},\\
Y_{l}(a)=Y_{l-1}(a)+\alpha_{l,L}W_{l}\:\phi(Y_{l-1}(a)),\quad l\in[L],
\end{cases}
\end{equation}
where $\phi$ is the ReLU activation function,\footnote{ReLU can be replaced with any polynomialy bounded activation function. We only consider ReLU here because it yields analytical expressions for the infinite-width-and-depth limits, see \cref{thm:main_Normalized}.} $W_{in}\in\reals^{n\times d}$
is the input layer weight matrix, and $W_{l}\in\reals^{n\times n}$ is
the weight matrix in the $l^{th}$ layer. We assume that the weights
are randomly initialized as $W_{d}^{ij}\sim\mathcal{N}(0,1/d)$, and
$W_{l}^{ij}\sim\mathcal{N}(0,1/n)$ for $l\in[L]$, $a\neq0$ is an
arbitrary input in $\reals^{d}$, $\alpha=(\alpha_{l,L})_{L\geq1,l\in[L]}$
is a sequence of scaling factors. For the sake of simplification,
we only consider networks with no bias, and we omit the dependence
of $Y_{l}$ on $n$ and $L$ in the notation. For a vector $Z\in\reals^{k}$,
we write $Z=(Z^{1},Z^{2},\dots,Z^{k})\in\reals^{k}$ to denote its
entries. Hereafter, we consider two inputs $a, b \in \reals^d$ satisfying $a,b \neq 0$ and $\langle a, b\rangle \neq 0$.\footnote{These conditions on $a,b$ are generally satisfied in practical scenarios. From a theoretical standpoint, we added these conditions in order to avoid dealing with division by $0$ etc. These cases are trivial and can be easily incorporated in the main results. However, we believe this is an unnecessary complication that does not add any value to the results.}

As depth increases, the pre-activations might grow arbitrarily large, depending on the choice of the sequence $\alpha$. The next result fully characterizes sequences that guarantee stability in terms of the $L_2$ norm.

\begin{lemma}\label{lemma:variance}
For all $L\geq 1, l\in[L], i \in [n]$ 
\[
\E\left[Y_{l}^{i}(a)^{2}\right]=\frac{\|a\|^{2}}{d}\prod_{k=1}^{l}\left(1+\frac{\alpha_{k,L}^{2}}{2}\right).
\]
As a result, $\sup_{l\in [L], L\geq 1, i \in [n]} \E\left[Y_{l}^{i}(a)^{2}\right]$ is bounded iff $\sup_{L\geq 1}\sum_{l=1}^L \alpha_{l,L}^2 < \infty$. \footnote{This stability condition was introduced in \cite{hayou21stable}. It is worth mentioning that in \cite{cirone2023neural}, in the context of ``controlled" ResNets, the authors use an $L_2$ condition on the control process to show weak commutativity for the output distribution where convergence is considered in the weak sense. This condition is similar in flavour to our stability condition. }
\end{lemma}
\begin{proof}
    Simple calculations yield 
    $$
    \E\left[Y_{l}^{i}(a)^{2}\right] = \E\left[Y_{l-1}^{i}(a)^{2}\right] + \alpha_{l,L}^2 \E\left[\phi(Y_{l}^{i}(a))^{2}\right].
    $$
    To conclude, it suffices to see that $Y_{l}^{i}(a)^{2}$ is a symmetric random variable, and therefore $\E\left[\phi(Y_{l}^{i}(a))^{2}\right] = \frac{1}{2} \E\left[Y_{l}^{i}(a)^{2}\right]$.
\end{proof}
The result of \Cref{lemma:variance} is independent from the width $n$. Hence, a necessary and sufficient condition so that the pre-activations do not blow up with depth (in $L_2$ norm), for any width $n$, is to have $\sup_{L\geq1}\sum_{l=1}^{L}\alpha_{l,L}^{2}<\infty.$
We say that such sequences of scaling factors are stable. 
\begin{definition}[Stable Sequence of Scaling Factors] Let $\alpha$ be a sequence of
scaling factors.  We say that $\alpha$ is stable if it satisfies $\sup_{L\geq1}\sum_{l=1}^{L}\alpha_{l,L}^{2}<\infty$. We denote the space of stable sequences of scaling factors by $\mathcal{S}$. For $\alpha \in \mathcal{S}$, we define the $\mathcal{S}$-norm of $\alpha$ by $\|\alpha\|_{S}=\sqrt{\sup_{L\geq1}\sum_{l=1}^{L}\alpha_{l,L}^{2}}$.\footnote{If we allow negative values for $\alpha_{l,L}$, then we can show that the space $\mathcal{S}$, endowed with the inner product $\langle\alpha,\beta\rangle_{S}=\sup_{L\geq1}\sum_{l=1}^{L}\alpha_{l,L}\beta_{l,L}$
, is a complete space (Banach space). We omit these technicalities in this paper.}
\end{definition}

Stable Sequences of Scaling Factors have first appeared in \cite{hayou21stable}. In that work, the sequential limit `infinite-width, then infinite-depth' was considered, and such sequences were proven to stabilize the gradients as well, and yield other favorable network properties regarding the neural covariance kernel and the neural tangent kernel.

In the next two (sub)sections, we show that  unlike in MLPs or residual networks with scaled main branch where the neural covariance/correlation exhibits different limiting behaviors depending on how the width and depth limits are taken, under general conditions on the sequence $\alpha$, for the ResNet architecture given by \cref{eq:main_resnet}, the neural covariance converges strongly to a deterministic kernel, which depends on the choice of the sequence $\alpha$, in the limit $\min(n,L) \to \infty$ regardless of the relative rate at which $n$ and $L$ tend to infinity. We show different examples and recover and strengthen previous results as special cases.

\subsection{Sequence of Scaling Factors as Convergent Series}
In this section, we consider sequences $\alpha$ that ``converge'' to a series in a specific way. We show that in this case, the neural covariance kernel converges to the same limiting kernel with a specific convergence rate in the limit $\min(n,L)\to \infty$, hence inducing commutativity.
 
\begin{thm}[Commutativity with Quasi-Convergent Series]\label{thm:main_ConvSeries}
Let $\alpha \in \mathcal{S}$. Assume that
there exists a sequence $\zeta=(\zeta_{i})_{i\geq1}\in\ell_{2}(\mathbb{N})$
such that $\sum_{l=1}^{L}|\alpha_{l,L}^{2}-\zeta_{l}^{2}|\to0$ as
$L\to\infty$. Then, we have that for all $t\in[0,1]$ 

\[
\sup_{t\in(0,1]}\|q_{\lfloor tL\rfloor,n}(a,b)-q^{\zeta}_\infty(a,b)\|_{L_{2}}\le C\left(n^{-1/2}+\sum_{l=1}^{L}|\alpha_{l,L}^{2}-\zeta_{l}^{2}|+\sum_{l\geq L}\zeta_{l}^{2}\right),
\]
where $C$ is a constant that depends only on $\|a\|,\|b\|,d,\|\zeta\|_{S}$, and $q^\zeta_\infty(a,b)=\lim_{L\to \infty}q^\zeta_L(a,b)$ and $q_L^\zeta$ is given by the recursive formula
$$
\begin{cases}
q_L^\zeta(a,b) = q_{L-1}^\zeta(a,b) + \frac{1}{2}\zeta_L^2 \frac{f(c_{L-1}(a,b))}{c_{L-1}(a,b)}q_{L-1}(a,b),\quad L\geq 1\\
c_L(a,b)=\frac{q_L(a,b)}{\sqrt{q_L(a,a)q_L(b,b)}},\\
q_0^\zeta(a,b)=\frac{\langle a,b\rangle}{d},
\end{cases} 
$$
where $f:[-1,1]\to[-1,1]$ is given by
$$
f(z)=\frac{1}{\pi}(z \arcsin{z} + \sqrt{1-z^2})+\frac{1}{2}z.
$$
\end{thm}
\Cref{thm:main_ConvSeries} shows that the neural covariance kernel converges to the same limiting kernel no matter how the width and depth limits are taken. In the proof, provided in \cref{app:infinite_depth_proofs}, we first show the existence of the limit of $q_L$, then proceed to bound the difference with the neural covariance kernel. The convergence rate depends on he properties of the series $\zeta$ that approximates $\alpha$ as depth grows. Notice that the limiting kernel $q^\zeta_\infty$ does not depend on $t \in (0,1]$. This is because the entries of $\zeta$ do not depend on depth $L$.

\textit{Examples.} The conditions of \Cref{thm:main_ConvSeries} are satisfied by many sequences $\alpha$. Examples include:
\begin{itemize}
    \item ``Decreasing'' scaling: assume that $\alpha_{l,L}=\zeta_l$ for all $L\geq 1, l\in [L]$, where $\zeta \in \ell_2(\mathbb{N})$. We call this scaling decreasing because $\lim_{l\to\infty} \zeta_l=0$. This choice of scaling factors trivially satisfies the conditions of \cref{thm:main_ConvSeries} and the convergence rate is given by $\bigO(n^{-1}+\sum_{l\geq L}\zeta_l^2)$. An examples of such scaling was studied in \citep{hayou21stable} and empirical results (performance of trained networks) were reported with $\zeta = \left((l \log(l+1)^2)^{-1/2}\right)_{l\geq1}$.

    \item ``Aggressive'' Uniform scaling: assume that $\alpha_{l,L}=L^{-\gamma}$ for some constant $\gamma > 1/2$. This scaling is called uniform because all the residual branches have the same scaling factor. This sequence of scaling factors satisfies the conditions of \cref{thm:main_ConvSeries} with $\zeta=0_{\ell_2(\mathbb{N})}$. The convergence rate is given by $\bigO(n^{-1}+L^{-(2\gamma-1)})$, and the limiting kernel is trivial and given by $q^\zeta_\infty=q_0^\zeta$, hence the nomenclature `aggressive' since this scaling removes all contributions of the hidden layers in the limiting kernel. Note that this case covers the Neural ODE limit with scaling factors $\alpha_{l,L}=L^{-1}$. In the next section, we will see that another kind of uniform scaling(non-aggressive) that yield non-trivial limits.
\end{itemize}

\subsection{Normalized Sequences of Scaling Factors}
In this section, we discuss another type of sequences of scaling factors. We know from \citep{Hayou2023WidthDepth} that with $\alpha_{l,L}=L^{-1/2}$, the limiting kernel is given by the solution of an ODE. In this section, we generalize this result by considering all sequences $\alpha$ that satisfy the condition $\sum_{l=1}^L\alpha_{l,L}^2 = 1$ for all $L\geq 1$. Let us first give a formal definition of such sequences.

\begin{definition}[Normalized Sequence of Scaling Factors] Let $\alpha$ be a sequence
of scaling factors. We say that $\alpha$ is normalized if it satisfies
$\sum_{l=1}^{L}\alpha_{l,L}^{2}=1$ for all $L\geq1$. The space of normalized sequences of scaling factors is denoted by $\sss_1$.\footnote{Note that the constant $1$ in this definition can  be replaced by any constant $M>0$ and all the subsequent results remain valid.}
\end{definition}
It is trivial that $\sss_1 \subset \sss$, and for all $\alpha \in \sss_1, \|\alpha\|_\sss=1$ (hence the subscript in $\sss_1$). The next result establishes commutativity of the infinite width and depth limit for normalized sequences.

\begin{thm}[Commutativity with Normalized Sequences]\label{thm:main_Normalized}
Consider a sequence of scaling factors $\alpha \in \sss_1$. Let $h_{L}=\max_{1\leq l\leq L}\alpha_{l,L}^{2}$ and assume that
$Lh_{L}^{2}=o_L(1).$ Then, we have 

\[
\sup_{t\in(0,1]}\|q_{\lfloor tL\rfloor,n}(a,b)-q_{t_L}(a,b)\|_{L_{2}}\le C\left(n^{-1/2}+h_L + L h_L^2\right),
\]
where $C$ depends only on $\|a\|,\|b\|,$ and $d$, and $t_L$ is given by $t_L = \sum_{k=1}^{\lfloor t L\rfloor} \alpha_{k,L}^2$, and $q_t$ is given by the solution of the following differential flow
\begin{equation}
\begin{aligned}
\begin{cases}
\frac{d q_t(a,b)}{dt} &= \frac{1}{2} \frac{f(c_t(a,b))}{c_t(a,b)} q_t(a,b),\\
c_t(a,b) &= \frac{q_t(a,b)}{ \sqrt{q_t(a,a)} \sqrt{q_t(b,b)}},\\
q_0(a,b) &= \frac{\langle a, b \rangle}{d},
\end{cases}
\end{aligned}
\end{equation}
where the function $f: [-1,1] \to [-1,1]$ is given by 
$$
f(z) = \frac{1}{\pi} ( z \arcsin(z) + \sqrt{1 - z^2}) + \frac{1}{2}z.
$$

Moreover, assume that there exists a function $\lambda:[0,1]\to [0,1]$ such that the sequence $\alpha$ satisfies $\sup_{t \in [0,1]}\left|\sum_{k=1}^{\lfloor t L\rfloor} \alpha_{k,L}^2-\lambda(t)\right| \leq r_L$ and $\lim_{L \to\infty} r_L = 0$. Then, we have
\[
\sup_{t\in(0,1]}\|q_{\lfloor tL\rfloor,n}(a,b)-q_{\lambda(t)}(a,b)\|_{L_{2}}\le C'\left(n^{-1/2}+h_L + L h_L^2 + r_L\right),
\]
where $C'$ depends only on $\|a\|,\|b\|, d$.
\end{thm}

The result of \cref{thm:main_Normalized} generalizes previous results from \citep{Hayou2023WidthDepth} to arbitrary normalized sequences. Using this theorem, we recover those results by choosing $\alpha_{l,L}=L^{-1/2}$ and verifying the conditions in the theorem. In particular, with the new proof techniques developed in this paper, we obtain a stronger convergence rate for depth.

\begin{corollary}[Normalized Uniform Scaling]
Assume that $\alpha_{l,L}=L^{-1/2}$ for all $L\geq 1$ and $l \in [L]$. Then, the results of \cref{thm:main_Normalized} are satisfied with $\lambda(t)=t$, $r_L=L^{-1}$, and $h_L=L^{-1}$. As a result, we have that 
\[
\sup_{t\in(0,1]}\|q_{\lfloor tL\rfloor,n}(a,b)-q_{t}(a,b)\|_{L_{2}}\le C\left(n^{-1/2}+L^{-1}\right),
\]
where $C$ depends only on $\|a\|,\|b\|, d$, and $q_t$ is defined in \cref{thm:main_Normalized}.
\end{corollary}

\begin{proof}
With $\alpha_{l,L}=L^{-1/2}$, we trivially have $h_L=L^{-1}$ and $L h_L^{2} = L^{-1}$. Moreover, given $t \in (0,1]$ we have that $\sum_{k=1}^{\lfloor t L \rfloor}\alpha_{l,L}^2=\frac{\lfloor t L \rfloor}{L}$, and therefore $\left|\sum_{k=1}^{\lfloor t L \rfloor}\alpha_{l,L}^2 - t \right|\leq L^{-1}$.
\end{proof}

\paragraph{Better depth rate.} In this paper, we obtain a depth rate of order $L^{-1}$ in contrast to the $L^{-1/2}$ convergence rate reported in \citep{Hayou2023WidthDepth}. The reason lies in the differences of the proof techniques used to derive the results. The proof techniques in both results are essentially `orthogonal' in the following sense: in \citep{Hayou2023WidthDepth}, the proofs rely on taking the depth to infinity first, while controlling the effect of width at the same time. With this approach, the best depth rate one can obtain is $L^{-1/2}$ which is induced by the Euler disctretization error (note that with $\alpha_{l,L}=L^{-1/2}$, the ResNet behaves as the solution of a Stochastic Differential Equation (SDE) in the infinite depth limit when the width is fixed). However, in this work, we first take the width to infinity while controlling the depth. By doing this, all the randomness in the covariance is removed as $n \to \infty$, regardless of the depth $L$. As a result, by taking depth to infinity, we deal with deterministic dynamical systems instead of stochastic ones (the SDE case), in which case the Euler disctretization error is of order $L^{-1}$. We refer the reader to \cref{sec:proofs} for more details about the proof techniques.\\

\textit{Remark.} The normalized uniform scaling is optimal in terms of the depth-related error in \cref{thm:main_Normalized}. More precisely, the depth-related error is given by the term $\mathcal{R}_L(\alpha) = h_L + L h_L^2 + r_L$ up to constant $C$. Therefore, a natural question one might ask is: what properties should the sequence of scaling factors satisfy in order to minimize this error? Given a fixed depth $L$, this problem can be formulated as a constrained minimization problem
\begin{equation}\label{eq:constrained_op}
    \min_{\alpha \in \sss_1} \mathcal{R}_L(\alpha) = h_L + L h_L^2 + r_L,
\end{equation}
where the constraint is given by the fact that $\alpha \in \sss_1$.

\begin{lemma}\label{lem:optimal_depth_scaling}
    The normalized uniform scaling given by $\alpha_{l,L}=L^{-1/2}$ is a solution to problem \eqref{eq:constrained_op}.
\end{lemma}

To explain the intuition behind the result of \cref{lem:optimal_depth_scaling}, we first need to understand what each term in $\mathcal{R}_L$ represents. The first term $h_L$ is well-known in numerical methods and represents the Euler discretization (global) error. The second term $L h_L^2$ is a bound on the error between the Euler scheme of the ODE satisfied by $q_t$ and the actual neural covariance kernel from the finite depth network. The last term $r_L$ is induced by the behaviour of the scaling sequence as $L$ grows. If we consider just the sum of the first two terms, uniform scaling balances the two terms which should intuitively minimize that sum. It also happens that for this choice of scaling $r_L$ is of the same order as $h_L + L h_L^2$.

\section{Experiments and Practical Implications}\label{sec:experiments}
In this section, we validate our theoretical results with simulations on large width and depth residual neural networks of the form \cref{eq:main_resnet} with different choices of the sequence $\alpha$. 
\begin{figure}[h]
    \centering
    \includegraphics[width=0.7\linewidth]{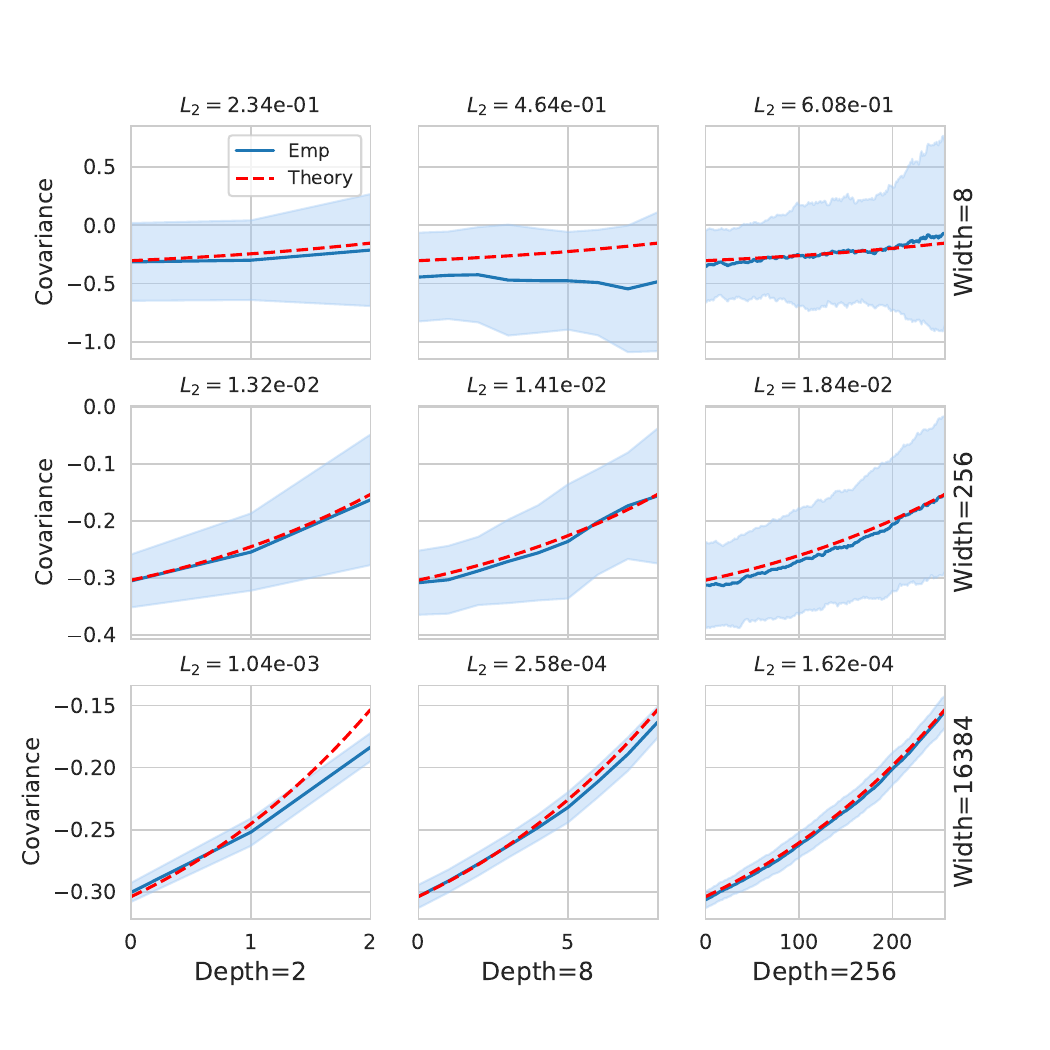}
    \caption{The blue curve represents the average covariance $q_{l,n}(a,b)$ for ResNet \cref{eq:main_resnet} with $n \in \{2^3, 2^8, 2^{14}\}$, $L\in \{2^1, 2^3, 2^8\}$, $d=30$, and $a$ and $b$ are sampled randomly from $\mathcal{N}(0, I_{d})$ and normalized to have $\|a\|=\|b\|=1$. The average is calculated based on $N=100$ simulations. The shaded blue area represents $1$ standard deviation of the observations. The red dashed line represents the theoretical covariance $q_t(a,b)$ predicted in \cref{thm:main_Normalized}. The empirical $L_2$ error for $t=1$ is reported.}
    \label{fig:covariance_uniform}
\end{figure}

\subsection{Convergence of the neural covariance}
\cref{thm:main_Normalized} and \cref{thm:main_ConvSeries} predict that the covariance $q_{l,n}(a,b)$ for two inputs $a, b$ converges in $L_2$ norm  in the limit $\min(n,L) \to \infty$.

\paragraph{Uniform scaling $\alpha_{l,L}=L^{-1/2}$.} In \cref{fig:covariance_uniform}, we compare the empirical covariance $q_{l,n}$ with the theoretical prediction $q_t$ from \cref{thm:main_Normalized} for $n \in \{2^3, 2^8, 2^{14}\}$ and $L \in \{2^1, 2^3, 2^8\}$. We chose maximum depth to be much smaller than maximum width to take into account the difference in the width and depth convergence rates: $n^{-1/2}$ versus $L^{-1}$ in this case.

The empirical $L_2$ error between $q_{L,n}$ and $q_1$ (from \cref{thm:main_Normalized}) is also reported. As the width increases, we observe an excellent match with the theory. The role of the depth is less noticeable, but for instance, with width $n = 2^14$, we can see that the $L_2$ error is smaller with depth $L=256$ as compared to depth $L=2$. The theoretical prediction $q_t$ is approximated with a PDE solver (RK45 method, \cite{Fehlberg1968ClassicalFS}) for $t \in [0,1]$ with a discretization step $\Delta t = $1e-6.

\begin{figure}[h]
    \centering
    \includegraphics[width=0.7\linewidth]{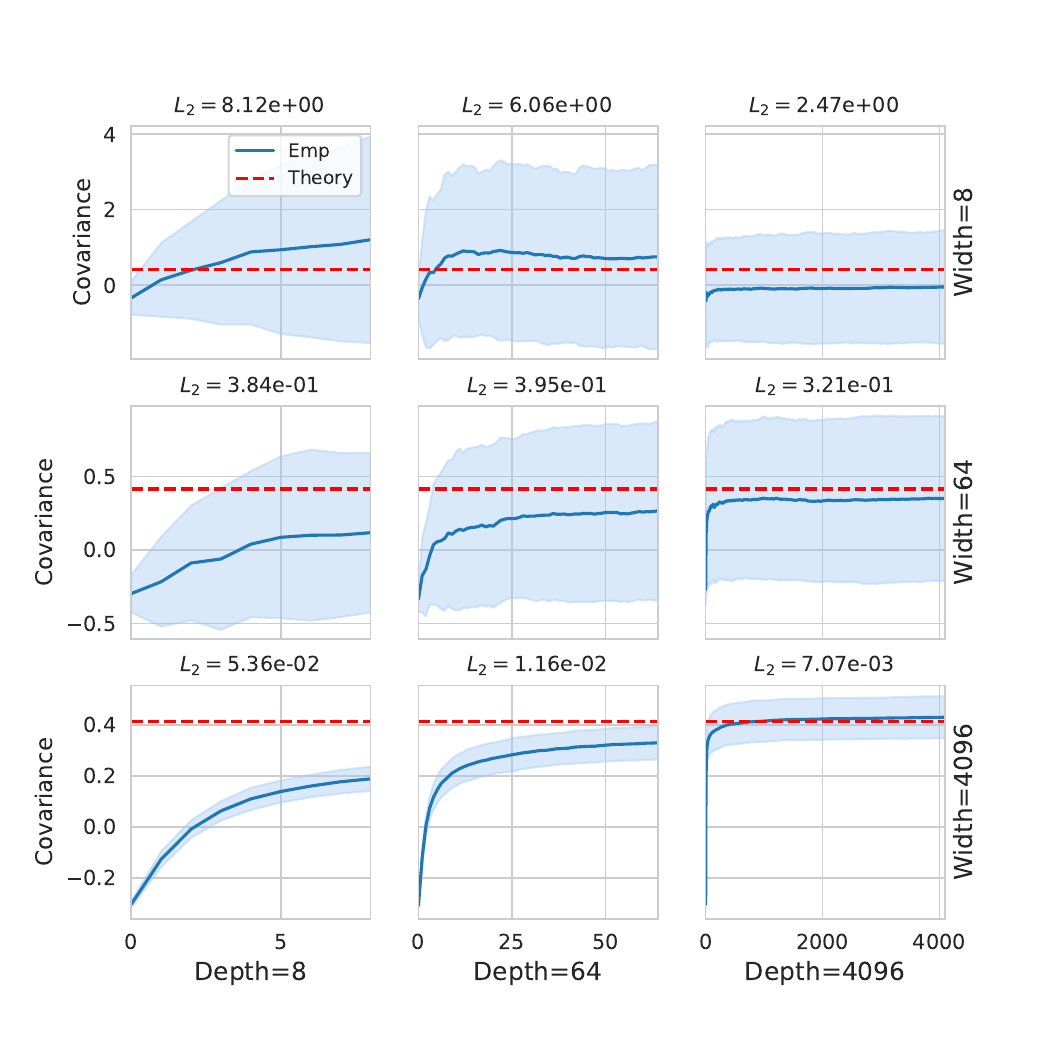}
    \caption{Same setup of \cref{fig:covariance_uniform}, with $\alpha_{l,L}=l^{-1}$. The red dashed line represents the theoretical covariance $q_\infty(a,b)$ predicted in \cref{thm:main_ConvSeries}. The empirical $L_2$ error for $t=1$ is reported.}
    \label{fig:covariance_decreasing}
\end{figure}

\paragraph{Convergent Scaling $\alpha_{l,L}=l^{-1}$.} In \cref{fig:covariance_decreasing}, we run the same experiment for decreasing scaling $\alpha_{l,L}=l^{-1}$. Note that in this case, the limiting neural covariance does not depend on $t$. The red dashed line represents this limiting value in the figure (estimated with $q_{L,n}$ with $L=$1e5 and $n=$1e5). Similar to the results with uniform scaling, we observe a convergence pattern to the red line as $L$ and $n$ increase. The role of depth in this case is more pronounced.

\begin{figure}
    \centering
    \includegraphics[width=0.33\linewidth]{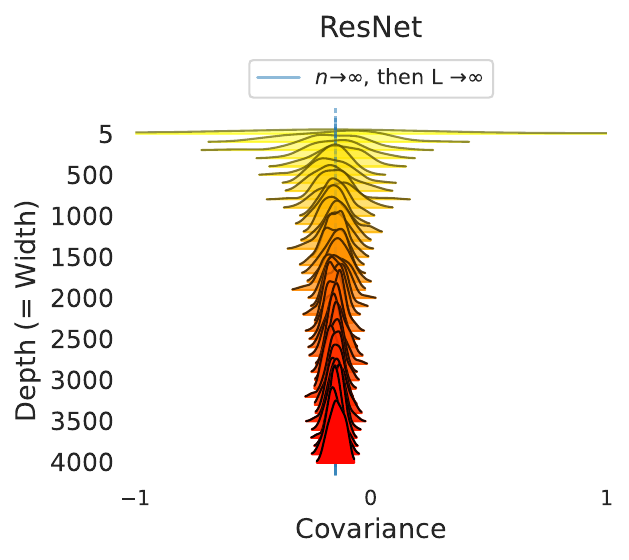}
    \includegraphics[width=0.3\linewidth]{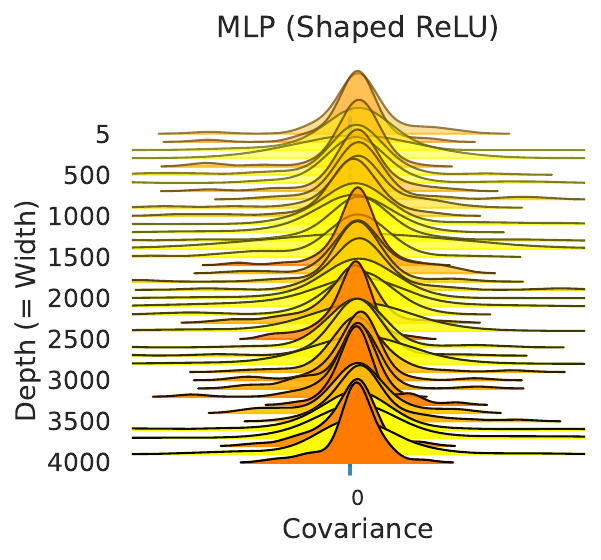}
    \includegraphics[width=0.3\linewidth]{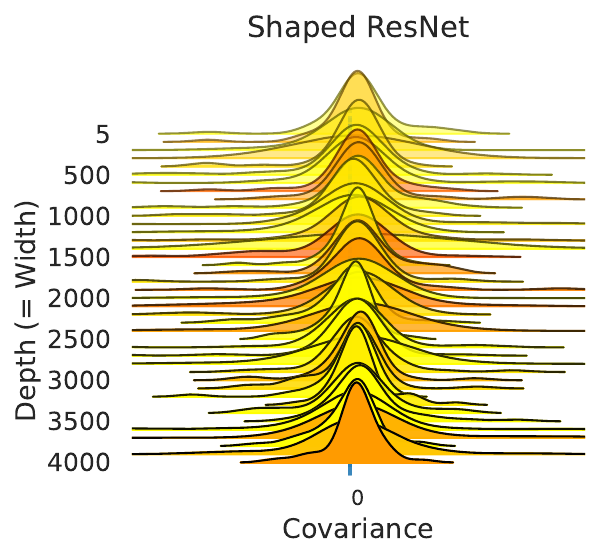}
    \caption{The distribution of the $q_{L,n}$ with $L=n$ for varying $n \in [5,4000]$. \textbf{(Left)} ResNet described in \cref{eq:main_resnet} with $\alpha_{l,L}=L^{-1/2}$. \textbf{(Center)} MLP described in \cref{eq:mlp} with shaped ReLU $\phi_L$ \cite{li2022sde}. \textbf{(Right)} Shaped ResNet with $\beta = 1/2$ (\cref{prop:shaped_resnet}). The vertical blue line represents the limit sequential limit $\lim_{L\to\infty}\lim_{n\to\infty}q_{L,n}$. The inputs $a,b$ are sampled following the same procedure in \cref{fig:covariance_uniform}.}
    \label{fig:compare_conv_covariance}
\end{figure}
\subsection{Comparison with other architectures}
In \cref{fig:compare_conv_covariance}, we show the evolution of the distribution of $q_{L,n}$ for three different architectures with $L=n$. With our choice of scaling factors $\alpha_{l,L}$, the distribution concentrates around the deterministic limit given by the solution of the ODE described in \cref{thm:main_Normalized}. For MLP with shaped ReLU, and the Shaped ResNet (\cref{prop:shaped_resnet}, the main branch is scaled with $\beta = 1/2$), we observe that the neural covariance remains random as width (and depth) grows. The sequential infinite-width-then-depth is illustrated in blue, and shows that with our choice of scaling factors, the covariance concentrates around this sequential limit even when $n=L \to \infty$. In contrast, with shaped MLP/ResNet, the two limits (sequential vs proportional) exhibit different behaviours, confirming that commutativity does not hold in these two cases.\\
\vspace{-2em}
\begin{wrapfigure}{r}{0.4\textwidth}
  \begin{center}
    \includegraphics[width=0.4\textwidth]{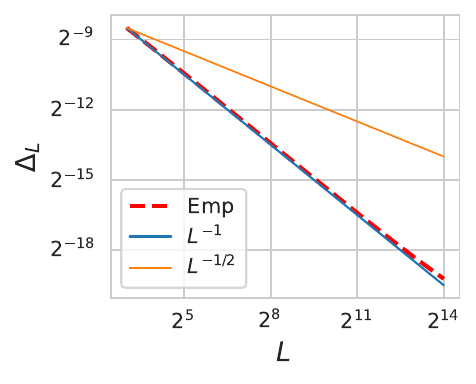}
  \end{center}
  \vspace{-0.75cm}
  \caption{\small{The curve of  $\Delta_L = |q_{L,\infty}(a,b) - q_{t=1}(a,b)|$ for $L \in \{2^k, k=3,\dots,14\}$.}}
  \vspace{-3em}
  \label{fig:depth_rate}
\end{wrapfigure}

\subsection{Improved Depth rate}
In \citep{Hayou2023WidthDepth}, commutativity was established with the choice of scaling factors $\alpha_{l,L}=L^{-1/2}$. The reported convergence rate (as $n$ and $L$ go to infinity) of the neural covariance is of the form $\bigO(n^{-1/2}+L^{-1/2})$. In this paper, we established an improved convergence rate of order $\bigO(n^{-1/2}+L^{-1})$, which suggests that convergence is more sensitive to the width than to depth. To validate this result, we conduct the following experiment: we take $n$ to infinity while fixing $L$ and obtain the infinite-width neural covariance $q_{L,\infty}$ (infinite-width covariance for the last layer). We then measure $\Delta_L = |q_{L,\infty}(a,b) - q_{t=1}(a,b)|$ where $q_t$ is given in \cref{thm:main_Normalized} and $(a,b)$ are sampled randomly following the procedure in \cref{fig:covariance_uniform}. We observe a perfect match of the $L^{-1}$ convergence rate (where the intercept was adjusted so that all the lines start from the same initial value).

\section{Outline of Proof Techniques}\label{sec:proofs}
In \citep{Hayou2023WidthDepth}, it was shown that the neural covariance satisfies commutativity with the specific scaling $\alpha_{l,L} = L^{-1/2}$. The main technical novelty in that work is taking the depth $L$ to infinity first, while controlling the dependence of the constants on the width $n$. Given a fixed width $n$, taking depth $L$ to infinity results in an SDE behaviour (Stochastic Differential Equation), and the main tools to study such convergence are numerical methods for SDEs (Euler discretization scheme). In this case, it is possible to obtain infinite-depth strong convergence where the constants do not depend on width $n$. Commutativity then follows by studying the infinite-width limit of these SDEs. This involves the use of tools from mean-field stochastic calculus (namely McKean-Vlasov processes).

In this work, we take an orthogonal approach where the width is taken to infinity first, and the constants are well chosen so that they do not depend on depth, followed by infinite-depth which concludes the proof. The main innovation in the proofs is related to the introduction of \emph{the auxiliary process} $\Tilde{Y}$: given a residual network of the form \cref{eq:main_resnet}, we introduce an auxiliary process $\tilde{Y}_l$ that shares some properties with the original neural process $Y_l$. We bound the difference between $Y_l$ and $\tilde{Y}_l$ using Gronwall's type of techniques, and show that the constants in this bound can be chosen to be independent of depth, hence providing a depth-uniform bound for the infinite-width limit. More importantly, the auxiliary process has iid Gaussian entries, which facilitates the study of the covariance kernel related to $\tilde{Y}_l$, and allow us to conclude on commutativity. Some technical results involve the use of concentration inequalities to deal with low probability events such as $\phi(Y_l)=0_{\reals^n}$ when $n$ is large.

\section{Conclusion and Limitations}

In this paper, we have shown that, at initialization, under general assumptions on the sequence of scaling factors, the large-depth and large-width limits of a residual neural network (resnet) commute for the neural covariance. We used novel proof techniques. Our results generalize and strengthen previous works on commutativity.

However, our results are restricted to the neural covariance function and cannot imply anything about commutativity for other neural functions. More importantly, it is unclear what happens during training, and potentially, different behaviors can occur depending on how the learning rate is chosen as a function of width and depth. While we can heuristically conjecture that commutativity holds during training under suitable scaling strategies, commutativity is a precise mathematical statement that requires rigorous proofs.

One might also ask whether commutativity is needed in the current context of Large Language Models, where most architectures are in the regime $n \gg L \gg 1$ (e.g. $n \sim 1000, L\sim50$) and that this regime can be fairly described by the sequential limit `$n \to \infty$, then $L \to \infty$'. While this might be true to some extent, note that convergence of neural functions can happen at different width and depth rates (e.g. $n^{-1/2}$ and $L^{-1}$ in the case of neural covariance), which implies that small changes in depth (or width) could significantly change the behaviour of the neural function. We leave this question for future work.

\newpage

\printbibliography
\newpage

\appendix

\addcontentsline{toc}{section}{Appendix} 

\part{Appendix} 

\parttoc

\newpage
\section{Infinite Width, Infinite Depth}\label{sec:comprehensive_lit_review}
Theoretical analysis of randomly initialized neural networks with an infinite number of parameters has yielded a wealth of interesting results, both theoretical and practical. Most of the research in this area has focused on the case where the network depth is fixed and the width is taken to infinity. However, in recent years, motivated by empirical observations, there has been an increased interest in studying the large depth limit of these networks. We provide here a non-exhaustive summary of existing results of these limits.

\subsection{Infinite-Width Limit}
The infinite-width limit of neural network architectures has been extensively studied in the literature and has led to many interesting theoretical and algorithmic innovations. We summarize these results below. 
\begin{itemize}[leftmargin=*]
    \item \emph{Initialization schemes}: the infinite-width limit of different neural architectures has been  extensively studied in the literature. In particular, for multi-layer perceptrons (MLP), a new initialization scheme that stabilizes forward and backward propagation (in the infinite-width limit) was derived in \citep{poole, samuel2017}. This initialization scheme is known as the Edge of Chaos, and empirical results show that it significantly improves performance. In \cite{yang2017meanfield, hayou21stable}, the authors derived similar results for the ResNet architecture, and showed that this architecture is \emph{placed} by-default on the Edge of Chaos for any choice of the variances of the initialization weights (Gaussian weights). In \cite{hayou2019impact}, the authors showed that an MLP that is initialized on the Edge of Chaos exhibits similar properties to ResNets, which might partially explain the benefits of the Edge of Chaos initialization.
    
    \item \emph{Gaussian process behaviour}: Multiple papers (e.g. \cite{neal, lee_gaussian_process, yang_tensor3_2020, matthews, hron20attention}) studied the weak limit of neural networks when the width goes to infinity. The results show that a randomly initialized neural network (with Gaussian weights) has a similar behaviour to that of a Gaussian process, for a wide range of neural architectures, and under mild conditions on the activation function. In \cite{lee_gaussian_process}, the authors leveraged this result and introduced the neural network Gaussian process (NNGP), which is a Gaussian process model with a neural kernel that depends on the architecture and the activation function. Bayesian regression with the NNGP showed that NNGP surprisingly achieves performance close to the one achieved by an SGD-trained finite-width neural network.
    
    The large depth limit of this Gaussian process was studied in \cite{hayou21stable}, where the authors showed that with proper scaling, the infinite-depth (weak) limit is a Gaussian process with a universal kernel\footnote{A kernel is called universal when any continuous function on some compact set can be approximated arbitrarily well with kernel features.}.
    \item \emph{Neural Tangent Kernel (NTK)}: the infinite-width limit of the NTK is the so-called NTK regime or Lazy-training regime. This topic has been extensively studied in the literature. The optimization and generalization properties (and some other aspects) of the NTK have been studied in \cite{Liu2022connecting, arora2019finegrained, seleznova2022ntk, hayou2019trainingdynamicsNTK}. The large depth asymptotics of the NTK have been studied in \citep{hayou_ntk, hayou2022curse, jacot2019freeze, xiao2020disentangling}. We refer the reader to \cite{jacot2022thesis} for a comprehensive discussion on the NTK.

    \item \emph{Feature Learning}: It is worth mentioning that the infinite-width limit was studied in the feature learning regime (as opposed to NTK regime where feature do not change). A series of works considered the infinite-width limit of mean-field parameterization, see e.g. \citep{bordelon2023influence, bordelon2022selfconsistent,fang2021modeling,sirignano2022mean,nguyen2020rigorous,araujo2019mean,mei2019mean,agoritsas2018out}. In the same direction, a series of works called \emph{Tensor Programs} studied the dynamics of infinite-width limit of finite-depth general neural networks both at initialization and at finite training step $t$ with gradient descent \citep{yang2019tensor_i, yang_tensor3_2020, yang2019scaling, yang2021tensor_iv}.
    
    \item \emph{Others}: the theory of infinite-width neural networks have also been utilized for network pruning  \citep{hayou_pruning}, regularization  \citep{vladimirova19understanding, hayou2021stochasticdepth}, feature learning \citep{hayou_eh}, and ensembling methods \citep{he2020ntkensembles}.
\end{itemize}

\subsection{Infinite-Depth Limit}

 \paragraph{Infinite-width-then-infinite-depth limit.} In this case, the width of the neural network is taken to infinity first, followed by the depth. This is known as the infinite-depth limit of infinite-width neural networks. This limit has been widely used to study various aspects of neural networks, such as analyzing neural correlations and deriving the Edge of Chaos initialization scheme \citep{samuel2017, poole}, investigating the impact of the activation function \citep{hayou2019impact}, and analyzing the behavior of the Neural Tangent Kernel (NTK) \citep{hayou_ntk, xiao2020disentangling}.
    
\paragraph{The joint infinite-width-and-depth limit.} In this case, the depth-to-width ratio is fixed\footnote{Other works consider the case when the depth-to-width ratio converge to a constant instead of being fixed.}, the width and depth are jointly taken to infinity. There are a limited number of studies that have examined the joint width-depth limit. For example, in \citep{li21loggaussian}, the authors demonstrated that for a specific form of residual neural networks (ResNets), the network output exhibits a (scaled) log-normal behavior in this joint limit, which is distinct from the sequential limit where the width is taken to infinity first followed by the depth, in which case the distribution of the network output is asymptotically normal (\citep{samuel2017, hayou2019impact}). Furthermore, in \citep{li2022sde}, the authors studied the covariance kernel of a multi-layer perceptron (MLP) in the joint limit and found that it weakly converges to the solution of a Stochastic Differential Equation (SDE). In \cite{Hanin2020Finite}, it was shown that in the joint limit case, the Neural Tangent Kernel (NTK) of an MLP remains random when the width and depth jointly go to infinity, which is different from the deterministic limit of the NTK when the width is taken to infinity before depth \citep{hayou_ntk}. In \citep{hanin2022correlation, hanin2019finitewidth}, the authors explored the impact of the depth-to-width ratio on the correlation kernel and the gradient norms in the case of an MLP architecture and found that this ratio can be interpreted as an effective network depth. Similar results have been discussed in \citep{zavatone2021exact, noci2021precise, jakub2023depth}.    
    
\paragraph{Infinite-depth limit of finite-width neural networks.}  In both previous limits, the width of the neural network is taken to infinity, either in isolation or jointly with the depth. However, it is natural to question the behavior of networks where the width is fixed and the depth is taken to infinity. For example, in \cite{hanin2019finitewidth}, it was shown that neural networks with bounded width are still universal approximators, motivating the examination of finite-width large depth neural networks. The limiting distribution of the network output at initialization in this scenario has been investigated in the literature. In \cite{peluchetti2020resnetdiffusion}, it was demonstrated that for a specific ResNet architecture, the pre-activations converge weakly to a diffusion process in the infinite-depth limit. This a simple corollary of existing results in stochastic calculus on the convergence of Euler-Maruyama disctretization schemes to continuous Stochastic Differential Equations. Other recent work by \cite{hayou2022on} examined the impact of the activation function on the distribution of the pre-activation, and characterized the distribution of the post-activation norms in this limit. It is worth mentioning that the neural ODE literature can be seen as part of the infinite-depth limit literature of neural networks, but we will not discuss that in this paper as we believe the two models (ODE model and our stochastic settings) are fundamentally different. 

\paragraph{General limit $\min\{n,L\} \to \infty$.} This limit is understudied, and to the best of our knowledge, it was only studied in \cite{Hayou2023WidthDepth}. This limit is particularly used to define (strong) commutativity.

\section{Depth-Uniform Infinite Width Limit: The Auxiliary Process}

In this section, we aim to understand the infinite-width behaviour
of the pre-activations $Y_{l}$ as a function of depth $L$. We will
show that there exists a process $\tilde{Y}_{l}(.):\reals^{d}\to\reals^{n}$
such that for any $a\in\reals^{d}$, the entries $(\tilde{Y}_{l}^{i}(a))_{i\in[n]}$
are iid Gaussian random variables, and 

$$n^{-1}\mathbb{{E}}\|Y_{l}(a)-\tilde{Y}_{l}(a)\|^{2}\leq C\sum_{i=1}^{l}\alpha_{l,L}^{2},$$where $C>0$ is a constant that depends only on the input $a$, and
which can be made independent of $a$ if the input is chosen in a compact
set. A straightforward result is that if the sequence $\alpha$ satisfies
$\sup_{L\geq1}\sum_{l=1}^{L}\alpha_{l,L}^{2}\leq M$ for some constant
$M$, then the convergence rate of the neural processes $Y_{l}$ to
$\tilde{Y}_{l}$ can be upperbounded by a quantity that does not depend
on depth.

\subsection{Constructing $\tilde{Y}_{l}$}

We can write the forward propagation as follows

$$Y_{l}(a)=Y_{l-1}(a)+\alpha_{l,L}\frac{1}{\sqrt{n}}\|\phi(Y_{l-1}(a))\|G_{l}(a)$$

$$G_{l}(a)=\begin{cases}
\sqrt{n}W_{l}\frac{\phi(Y_{l-1}(a))}{\|\phi(Y_{l-1}(a))\|} & \textrm{if}\hspace{1em}\|\phi(Y_{l-1}(a))\|\neq0,\\
\sqrt{n}W_{l}e & \textrm{otherwise,}
\end{cases}$$

where $e=(1,\dots,1)^{\top}\in\reals^{n}$ (the choice of $e$ here
is arbitrary and does not impact the identity above). The
vector $G_{l}(a)$ consists of iid standard Gaussian variables as
a result of \cref{lemma:gaussian_vec}. Moreover, for any $l\neq l'$, the
processes $G_{l}$ and $G_{l^{'}}$ are independent.

Using this auxiliary process $G_{l}$, we define the process $\tilde{Y}_{l}$
as follows

$$\tilde{Y}_{l}(a)=\tilde{Y}_{l-1}(a)+\alpha_{l,L}\left(\E[\phi(\tilde{Y}_{l-1}^{1}(a))^{2}]\right)^{1/2}G_{l}(a)$$ 

The $\emph{volatility}$ term $\E[\phi(\tilde{Y}_{l-1}^{1}(a))^{2}]$
in the definition of the process $\tilde{Y}_{l}$ can be expressed
analytically. We state this result in the next lemma.
\begin{lemma}
\label{lem:variance}For all $l\in[L],$ 
\[
q_{l}(a)=\E[(Y_{l}^{1})^{2}]=\E[(\tilde{Y_{l}^{1}})^{2}]=\frac{\|a\|^{2}}{d}\prod_{k=1}^{l}\left(1+\frac{\alpha_{k,L}^{2}}{2}\right).
\]
 As a result, we also have 
\[
\E\left[\phi(Y_{l}^{1}(a))^{2}\right]=\E\left[\phi(\tilde{Y}_{l}^{1}(a))^{2}\right]=\frac{\|a\|^{2}}{2d}\prod_{k=1}^{l}\left(1+\frac{\alpha_{k,L}^{2}}{2}\right).
\]
.
\end{lemma}
\begin{proof}
Simple calculations yield.
\begin{align*}
\E[(Y_{l}^{1})^{2}]&=\E[(Y_{l-1}^{1})^{2}]+\alpha_{l,L}^{2}\E[\phi(Y_{l-1}^{1})^{2}]\\
&=\left(1+\frac{\alpha_{l,L}^{2}}{2}\right)\E[(Y_{l-1}^{1})^{2}].
\end{align*}
Knowing that $\E[(Y_{0}^{1})^{2}]=\frac{\|a\|^{2}}{d}$, we obtain
$\E[(Y_{l}^{1})^{2}]=\frac{\|a\|^{2}}{d}\prod_{k=1}^{l}\left(1+\frac{\alpha_{l,L}^{2}}{2}\right)$.
Similar calculations hold for $\E[(\tilde{Y_{l}^{1}})^{2}]$.

From Lemma \ref{lem:variance}, we can write the process
$\tilde{Y}_{l}$ by substituting the volatility term with its analytical
expression. This allows us to conclude that $\tilde{Y}_{l}$ has iid
Gaussian weights with a analytical expression of the variance.
\end{proof}

\begin{lemma}
The process $\tilde{Y}_{l}$ satisfies the following 

\[
\tilde{Y}_{l}(a)=\tilde{Y}_{l-1}(a)+\alpha_{l,L}\frac{\|a\|}{\sqrt{2d}}\prod_{k=1}^{l}\left(1+\frac{\alpha_{k,L}^{2}}{2}\right)^{1/2}G_{l}(a).
\]
As a result, the entries of $\tilde{Y}_{l}(a)$ are iid centered Gaussian
random variables with variance $\textrm{Var}(\tilde{Y}_{l}^{1}(a))=\frac{\|a\|^{2}}{d}\prod_{k=1}^{l}\left(1+\frac{\alpha_{l,L}^{2}}{2}\right)$.
\end{lemma}
Note that while the the entries of $\tilde{Y}_{l}(a)$ are Gaussian,
the process $\tilde{Y_{l}(.)} $ is not necessarily a Gaussian process, although it can be proven that it converges to a Gaussian process in the infinite-width limit.

\subsection{Convergence Rate}

In this section, we will analyze the convergence properties of different
quantities as width goes to infinity.

\begin{thm}
\label{thm:Depth-Uniform-strong-convergence}(Depth-Uniform strong
convergence rate) Let $\alpha$ be a stable sequence of scaling factors.
Then, there exists a constant $C>0$ that depends only on $\|a\|,d,\|\alpha\|_{S}$
such that 

\[
\sup_{L\geq1}\sup_{l\in[L]}\mathbb{{E}}\|Y_{l}(a)-\tilde{Y}_{l}(a)\|^{2}\leq C.
\]
As a result, we have that 
\[
\sup_{L\geq1}\sup_{l\in[L]}\sup_{i\in[n]}\mathbb{{E}}\|Y_{l}^{i}(a)-\tilde{Y}_{l}^{i}(a)\|^{2}\leq Cn^{-1}.
\]
\end{thm}
\begin{proof}
Let $a\in\reals^{d}.$ To alleviate the notation, we write $Y_{l}:=Y_{l}(a)$
and $\tilde{Y}_{l}:=\tilde{Y}_{l}(a).$ Now we would like to obtain
recursive bounds $\mathbb{{E}}\|Y_{l}-\tilde{Y}_{l}\|^{2}$ which
will allow us to conclude. The proof technique follows Gronwall's
style inequalities. We have the following

\[
\mathbb{{E}}\|Y_{l}-\tilde{Y}_{l}\|^{2}=\mathbb{{E}}\|Y_{l-1}-\tilde{Y}_{l-1}\|^{2}+n\alpha_{l,L}^{2}\underbrace{\E\left(\frac{1}{\sqrt{n}}\|\phi(Y_{l-1})\|-\left(\E[\phi(\tilde{Y}_{l-1}^{1})^{2}]\right)^{1/2}\right)^{2}}_{T}.
\]

We bound the second term $T$ as follows

\[
\begin{aligned}\E\left(\frac{1}{\sqrt{n}}\|\phi(Y_{l-1})\|-\left(\E[\phi(\tilde{Y}_{l-1}^{1})^{2}]\right)^{1/2}\right)^{2} & \leq2\E\left(\frac{1}{\sqrt{n}}\|\phi(Y_{l-1})\|-\frac{1}{\sqrt{n}}\|\phi(\tilde{Y}_{l-1})\|\right)^{2}\\
 & +2\E\left(\frac{1}{\sqrt{n}}\|\phi(\tilde{Y}_{l-1})\|-\left(\E[\phi(\tilde{Y}_{l-1}^{1})^{2}]\right)\right)^{2}\\
 & \leq\frac{2}{n}\E\|Y_{l-1}-\tilde{Y}_{l-1}\|^{2}+2\E\left(\frac{1}{\sqrt{n}}\|\phi(\tilde{Y}_{l-1})\|-\left(\E[\phi(\tilde{Y}_{l-1}^{1})^{2}]\right)\right)^{2}
\end{aligned}
\]
where we have used the fact that $\phi$ is 1-Lipschitz. 

Knowing that the entries of $\tilde{Y}_{l-1}$ are iid, and
that $q_{l}(a)\in\left[\frac{\|a\|^{2}}{d},\frac{\|a\|^{2}}{d}e^{\frac{1}{2}\|\alpha\|_{S}^{2}}\right],$
standard concentration inequalities (Hoeffding's inequality) ensure
that with probability at least $1-e^{-nc}$ (where $c$ is a constant
that depends only on $\|a\|,\|\alpha\|_{S}$,$d$), we have that

$\frac{1}{n}\|\phi(\tilde{Y}_{l-1})\|^{2}>\frac{1}{2}\E[\phi(\tilde{Y}_{l-1}^{1})^{2}]=\frac{1}{4}q_{l-1}(a)$. 
Using this results combined with the fact that $|\sqrt{x_{1}}-\sqrt{x_{2}}|\leq\frac{1}{2\sqrt{x_{0}}}|x_{1}-x_{2}|$
for all $x_{1},x_{2}>x_{0}>0,$ we obtain 

\[
\begin{aligned}\E\left(\frac{1}{\sqrt{n}}\|\phi(\tilde{Y}_{l-1})\|-\left(\E[\phi(\tilde{Y}_{l-1}^{1})^{2}]\right)\right)^{2} & \leq2e^{-nc}q_{l-1}(a)+\frac{2}{q_{l-1}(a)}\E\left(\frac{1}{n}\|\phi(\tilde{Y}_{l-1})\|^{2}-\E[\phi(\tilde{Y}_{l-1}^{1})^{2}]\right)^{2}\\
 & \leq2e^{-nc}q_{l-1}(a)+\frac{2}{nq_{l-1}(a)}\E[\phi(\tilde{Y}_{l-1}^{1})^{4}]\\
 & \leq2e^{-nc}q_{l-1}(a)+\frac{2q_{l-1}(a)}{n}\E[\phi(Z)^{4}]
\end{aligned}
\]

where $Z\sim\normal(0,1).$ As a result, there exists a constant $C_{1}>0$
that depends only on $\|a\|,d,$ and $\|\alpha\|_{S}$, such that 

$$\E\left(\frac{1}{\sqrt{n}}\|\phi(\tilde{Y}_{l-1})\|-\left(\E[\phi(\tilde{Y}_{l-1}^{1})^{2}]\right)\right)^{2}\leq C_{1}n^{-1}.$$
Hence, denoting $\Delta_{l}=n^{-1}\mathbb{{E}}\|Y_{l}-\tilde{Y}_{l}\|^{2}$,
we have that 
\[
\Delta_{l}\leq(1+2\alpha_{l,L}^{2})\Delta_{l-1}+2C\alpha_{l,L}^{2}n^{-1}.
\]
Given that $\Delta_{0}=0$, we obtain 

\[
\Delta_{l}\leq n^{-1}\times2C_{1}\sum_{i=1}^{l}\alpha_{i,L}^{2}\prod_{k=i+1}^{l}(1+\alpha_{k,L}^{2})\leq2C_{1}\|\alpha\|_{S}^{2}e^{\|\alpha\|_{S}^{2}}n^{-1},
\]
 which concludes the proof. 
\end{proof}
As a result of this theorem, we have the following result (a useful
lemma for subsequent proofs).
\begin{lemma}
\label{lem:hl}Let $a\in\reals^{d},\zeta\in[0,8^{-1/2}d^{-1/2}\|a\|).$
For $L\geq1$ and $l\in[L]$, define the event 
\[
\mathcal{H}_{a}^{l}=\{\|\phi(Y_{l}(a))\|>\zeta n^{1/2}\}\cap\{\|\phi(\tilde{Y}_{l}(a))\|>\zeta n^{1/2}\}.
\]
Then, we have that $\mathbb{P}(\mathcal{H}_{a}^{l})\geq1-Cn^{-1}$,
where $C$ is a constant that depends only on $\|a\|,d,\|\alpha\|_{S}$.
\end{lemma}
\begin{proof}
We have that $\mathcal{H}_{a}^{l}=E_{a}\cap\tilde{E}_{a}$, where
$E_{a}=\{\|\phi(Y_{l}(a))\|>\zeta n^{1/2}\}$, and $\tilde{E}_{a}=\{\|\phi(\tilde{Y}_{l}(a))\|>\zeta n^{1/2}\}$.
For some event $A$, let $A^{c}$ denote its complimentary event.
Using the fact that the entries of $\tilde{Y_{l}}(a)$ are iid zero-mean
Gaussians, we have that 
\begin{align*}
\mathbb{P}(\tilde{E}_{a}^{c}) & =\mathbb{P}(\|\phi(\tilde{Y}_{l}(a))\|\leq\zeta n^{1/2})=\mathbb{P}(\frac{1}{n}\|\phi(\tilde{Y}_{l}(a))\|^{2}\leq\zeta^{2})\\
 & \leq\mathbb{P}(\frac{1}{n}\|\phi(\tilde{Y}_{l}(a))\|^{2}\leq\frac{1}{4}q_{l}(a))\\
 & \leq e^{-nC_{1}}
\end{align*}
where $C_{1}$ is a constant that depends only on $\|a\|,d,\|\alpha\|_{S},\zeta$,
where we have used the same techniques as in the proof of \cref{thm:Depth-Uniform-strong-convergence}
(Hoeffding's inequality).

Now let $\kappa=\left(\frac{q_{l}(a)}{8}n\right)^{1/2}$ . We have
that 
\begin{align*}
\mathbb{P}(E_{a}^{c}) & =\mathbb{P}(\|\phi(Y_{l}(a))\|\leq\zeta n^{1/2})\leq\mathbb{P}(\|\phi(\tilde{Y}_{l}(a))\|\leq\kappa+\zeta n^{1/2})\\
 & \quad\quad\quad\quad\quad+\mathbb{P}(\|\phi(Y_{l}(a))-\phi(\tilde{Y}_{l}(a))\|>\kappa).
\end{align*}
Therefore, we obtain

\[
\mathbb{P}(\|\phi(\tilde{Y}_{l}(a))\|\leq\kappa+\zeta n^{1/2})\leq\mathbb{P}(\frac{1}{n}\|\phi(\tilde{Y}_{l}(a))\|\leq\frac{1}{4}q_{l}(a))\leq e^{-nC_{1}}.
\]
 Using \cref{thm:Depth-Uniform-strong-convergence}, Markov's inequality,
and the fact that $\phi$ is $1$-Lipschitz, we have that
\begin{flushleft}
\begin{align*}
\mathbb{P}(\|\phi(Y_{l}(a))-\phi(\tilde{Y}_{l}(a))\| & >\kappa)\leq\kappa^{-2}\E\|\phi(Y_{l}(a))-\phi(\tilde{Y}_{l}(a))\|^{2}\\
 & \leq\frac{4K}{q_{l}(a)}n^{-1} \leq\frac{4Kd}{\|a\|^{2}}n^{-1}.
\end{align*}
\par\end{flushleft}
Combining both bounds, there exists a constant $C_{2}$ that depends
only on $\|a\|,d,\|\alpha\|_{S},\zeta$, such that $\mathbb{P}(E_{a}^{c})\leq C_{2}n^{-1}$.
\end{proof}

\subsection{Infinite-Width Limits of the Neural Covariance}

The auxiliary process $\tilde{Y}_{l}$ is introduced for two reasons:
\begin{enumerate}
\item The distance between $\tilde{Y}_{l}$ and $Y_l$ as $n$ grows can be upperbounded so that the constants do not depend on depth.
\item It is easier to study the covariance kernel of the $\tilde{Y}_{l}$ instead of that of $Y_l$ as $n$ and $l$ go to infinity. 
\end{enumerate}
We dealt with (1) in the previous section, now we deal with (2).

Define the covariance kernel of the auxiliary process 
$$\tilde{q}_{l,n}(a,b)=n^{-1}\langle\tilde{Y}_{l}(a),\tilde{Y}_{l}(b)\rangle.$$

This covariance kernel satisfies the following recursion
\begin{align*}
\tilde{q}_{l,n}(a,b) & =\tilde{q}_{l-1,n}(a,b)+\alpha_{l,L}^{2}n^{-1}(1/2q_{l-1}(a))^{1/2}(1/2q_{l-1}(b))^{1/2}\langle G_{l}(a),G_{l}(b)\rangle\\
 & +\alpha_{l,L}n^{-1}((1/2q_{l-1}(b))^{1/2}\langle\tilde{Y}_{l-1}(a),G_{l}(b)\rangle+(1/2q_{l-1}(a))^{1/2}\langle\tilde{Y}_{l-1}(b),G_{l}(a)\rangle)
\end{align*}

In the following, we will show that in the infinite-width limit, the
kernel $\tilde{q}_{l,n}$ converges to a kernel $\tilde{q}_{l,\infty}$
that satisfies the following recursion
$$
\tilde{q}_{l,\infty}(a,b)=\tilde{q}_{l-1,\infty}(a,b)+\alpha_{l,L}^{2}(1/2q_{l-1}(a))^{1/2}(1/2q_{l-1}(b))^{1/2}f(c_{l-1}(a,b)),
$$

where $f(c):=2\E[\phi(Z_{1})\phi(cZ_{1}+\sqrt{1-c^{2}}Z_{2})]$ with
$Z_{1,}Z_{2}\sim\normal(0,1),$ and $c_{l-1,\infty}(a,b):=\frac{\tilde{q}_{l-1,\infty}(a,b)}{\tilde{q}_{l-1,\infty}(a,a)^{1/2}\tilde{q}_{l-1,\infty}(b,b)^{1/2}}$
(the infinite-width correlation kernel).

\paragraph{Remark.} Observe that $\tilde{q}_{l-1,\infty}(a,a)=q_{l-1}(a).$ (proof is straightforward by induction).

Now we derive non-asymptotic convergence rates for the covariance kernel $\tilde{q}_{l,n}$ in the infinite-width limit. Similar to the analysis in the previous section, define the $L_{2}$ error between the kernels by $\tilde{\Delta}_{l,n}:=\E\left|\tilde{q}_{l,n}(a,b)-\tilde{q}_{l,\infty}(a,b)\right|^{2}$.
Simple calculations yield
\begin{center}
\begin{align*}
\tilde{\Delta}_{l,n} & =\tilde{\Delta}_{l-1,n}+\E(\alpha_{l,L}^{2}(1/2q_{l-1}(a))^{1/2}(1/2q_{l-1}(b))^{1/2}(n^{-1}\langle G_{l}(a),G_{l}(b)\rangle-f(c_{l-1}(a,b)))\\
&+\alpha_{l,L}n^{-1}((1/2q_{l-1}(b))^{1/2}\langle\tilde{Y}_{l-1}(a),G_{l}(b)\rangle+(1/2q_{l-1}(a))^{1/2}\langle\tilde{Y}_{l-1}(b),G_{l}(a)\rangle))^{2}\\
\leq & \tilde{\Delta}_{l-1,n}+\underbrace{\frac{1}{2}\alpha_{l,L}^{4}q_{l-1}(a)q_{l-1}(b)\E\left(n^{-1}\langle G_{l}(a),G_{l}(b)\rangle-f(c_{l-1}(a,b))\right)^{2}}_{T_{1}}\\
&+\underbrace{2\alpha_{l,L}^{2}n^{-2}\left(q_{l-1}(b)\E\langle\tilde{Y}_{l-1}(a),G_{l}(b)\rangle^{2}+q_{l-1}(a)\E\langle\tilde{Y}_{l-1}(b),G_{l}(a)\rangle^{2}\right)}_{T_{2}}
\end{align*}
\par\end{center}

We will deal with the diffferent terms separately. 

\paragraph{Bounding $T_{2}$:} We have that 

\begin{center}
\[
q_{l-1}(b)\:\E\langle\tilde{Y}_{l-1}(a),G_{l}(b)\rangle^{2}=q_{l-1}(b)\,\E\|\tilde{Y}_{l-1}(a)\|^{2}=nq_{l-1}(b)q_{l-1}(a)\leq\frac{\|a\|^{2}\|b\|^{2}}{d^{2}}e^{\|\alpha\|_{S}^{2}}
\]
\par\end{center}

As a result, we obtain
\begin{align*}
T_{2}&=2\alpha_{l,L}^{2}n^{-2}\left(q_{l-1}(b)\E\langle\tilde{Y}_{l-1}(a),G_{l}(b)\rangle^{2}+q_{l-1}(a)\E\langle\tilde{Y}_{l-1}(b),G_{l}(a)\rangle^{2}\right)\\
& \leq2\frac{\|a\|^{2}\|b\|^{2}}{d^{2}}e^{\|\alpha\|_{S}^{2}}\alpha_{l,L}^{2}n^{-1}.
\end{align*}

\paragraph{Bounding $T_{1}$:} Define the events $\mathcal{H}_{a}^{l}=\{\|\phi(Y_{l}(a))\|\neq0\}\cap\{\|\phi(\tilde{Y}_{l}(a))\|\neq0\}$
and $\mathcal{H}_{b}^{l}=\{\|\phi(Y_{l}(b))\|\neq0\}\cap\{\|\phi(\tilde{Y}_{l}(b))\|\neq0\}$.
We will condition on the event $\mathcal{H}_{a}^{l-1}\cap\mathcal{H}_{b}^{l-1}$ to avoid dividing by zero. This allows us to control a conditional
expectation in the following manner 
\begin{align*}
\E\left(n^{-1}\langle G_{l}(a),G_{l}(b)\rangle-f(c_{l-1}(a,b))\right)^{2}&\leq C_{1}n^{-1}\\
&+\E\left[\left(n^{-1}\langle G_{l}(a),G_{l}(b)\rangle-f(c_{l-1}(a,b))\right)^{2}\mid\mathcal{H}_{a}^{l-1}\cap\mathcal{H}_{b}^{l-1}\right],
\end{align*}
where $C_{1}$ is a constant that depends only on $\|a\|,\|b\|,d,\|\alpha\|_{S}$
(using \cref{lem:hl} with $\zeta=0$). To alleviate the notation,
we denote $\E_{l}[.]=\E[.\mid\mathcal{H}_{a}^{l-1}\cap\mathcal{H}_{b}^{l-1}].$
We therefore have

\begin{align*}
\E_{l}&\left(n^{-1}\langle G_{l}(a),G_{l}(b)\rangle-f(c_{l-1}(a,b))\right)^{2}\leq\\
&3\underbrace{\E_{l}\left(n^{-1}\langle G_{l}(a),G_{l}(b)\rangle-\E_{l}\frac{\langle\phi(Y_{l-1}(a)),\phi(Y_{l-1}(b))\rangle}{\|\phi(Y_{l-1}(a))\|\|\phi(Y_{l-1}(b))\|}\right)^{2}}_{T_{11}}\\
&+3\underbrace{\left(\E_{l}\frac{\langle\phi(Y_{l-1}(a)),\phi(Y_{l-1}(b))\rangle}{\|\phi(Y_{l-1}(a))\|\|\phi(Y_{l-1}(b))\|}-\E_{l}\frac{\langle\phi(\tilde{Y}_{l-1}(a)),\phi(\tilde{Y}_{l-1}(b))\rangle}{\|\phi(\tilde{Y}_{l-1}(a))\|\|\phi(\tilde{Y}_{l-1}(b))\|}\right)}_{T_{12}}^{2}\\
&+3\underbrace{\left(\E_{l}\frac{\langle\phi(\tilde{Y}_{l-1}(a)),\phi(\tilde{Y}_{l-1}(b))\rangle}{\|\phi(\tilde{Y}_{l-1}(a))\|\|\phi(\tilde{Y}_{l-1}(b))\|}-f(c_{l-1}(a,b))\right)^{2}}_{T_{13}}
\end{align*}

We deal with each one of the terms $T_{11}, T_{12}, T_{13}$ separately.

\begin{itemize}
    \item The first term $T_{11}$ satisfies 
\begin{align*}
T_{11}&=\E_{l}\left(n^{-1}\langle G_{l}(a),G_{l}(b)\rangle-\E_{l}\frac{\langle\phi(Y_{l-1}(a)),\phi(Y_{l-1}(b))\rangle}{\|\phi(Y_{l-1}(a))\|\|\phi(Y_{l-1}(b))\|}\right)^{2}\\
&=n^{-1}\textrm{Var}_{l}\left((w^{\top}\frac{\phi(Y_{l-1}(a))}{\|\phi(Y_{l-1}(a))\|})\times(w^{\top}\frac{\phi(Y_{l-1}(b))}{\|\phi(Y_{l-1}(b))\|})\right)\\
&\leq n^{-1}\E_{l}\left((w^{\top}\frac{\phi(Y_{l-1}(a))}{\|\phi(Y_{l-1}(a))\|})\times(w^{\top}\frac{\phi(Y_{l-1}(b))}{\|\phi(Y_{l-1}(b))\|})\right)^{2},
\end{align*}

where $w\sim\normal(0,I)$. We bound this quantity using the following
lemma.
\begin{lemma}
We have that $\E_{l}\left((w^{\top}\frac{\phi(Y_{l-1}(a))}{\|\phi(Y_{l-1}(a))\|})\times(w^{\top}\frac{\phi(Y_{l-1}(b))}{\|\phi(Y_{l-1}(b))\|})\right)^{2}\leq3,$
for $w\sim\normal(0,I)$.
\end{lemma}
\begin{proof}
Let $u(a):=\frac{\phi(Y_{l-1}(a))}{\|\phi(Y_{l-1}(a))\|}$ and the
same for $b$. Expanding the term inside the expectation yields $\left((w^{\top}u(a))\times(w^{\top}u(b))\right)^{2}=\sum_{i=1}w_{i}^{4}u_{i}(a)^{2}u_{i}(b)^{2}+\sum_{i\neq j}w_{i}^{2}w_{j}^{2}u_{i}(a)^{2}u_{j}(b)^{2}+\textrm{\ensuremath{\zeta}}$
, where $\zeta$ consists of terms with at least one weight having
an odd exponent. For such terms, the expectation is null, and we obtain
\begin{align*}
\E_{l}\left((w^{\top}u(a))\times(w^{\top}u(b))\right)^{2}&=\E_{l}\left(3\sum_{i=1}u_{i}(a)^{2}u_{i}(b)^{2}+\sum_{i\neq j}u_{i}(a)^{2}u_{j}(b)^{2}\right)\\
&=\E_{l}\left(2\sum_{i=1}u_{i}(a)^{2}u_{i}(b)^{2}+1\right)\leq3.
\end{align*}

\end{proof}

Using this Lemma, we obtain 
\[
T_{11}=\E_{l}\left(n^{-1}\langle G_{l}(a),G_{l}(b)\rangle-\E\frac{\langle\phi(Y_{l-1}(a)),\phi(Y_{l-1}(b))\rangle}{\|\phi(Y_{l-1}(a))\|\|\phi(Y_{l-1}(b))\|}\right)^{2}\leq3n^{-1}.
\]

\item Now let us deal with the second term $T_{12}$. We use the uniform bound we obtained in \cref{thm:Depth-Uniform-strong-convergence}.
We have that 
\begin{equation}\label{eq:T12}
\begin{aligned}
&\left(\E_{l}\frac{\langle\phi(Y_{l-1}(a)),\phi(Y_{l-1}(b))\rangle}{\|\phi(Y_{l-1}(a))\|\|\phi(Y_{l-1}(b))\|}-\E_{l}\frac{\langle\phi(\tilde{Y}_{l-1}(a)),\phi(\tilde{Y}_{l-1}(b))\rangle}{\|\phi(\tilde{Y}_{l-1}(a))\|\|\phi(\tilde{Y}_{l-1}(b))\|}\right)^{2}\\
&\leq2\left(\E_{l}\frac{\langle\phi(Y_{l-1}(a)),\phi(Y_{l-1}(b))\rangle}{\|\phi(Y_{l-1}(a))\|\|\phi(Y_{l-1}(b))\|}-\E_{l}\frac{\langle\phi(\tilde{Y}_{l-1}(a)),\phi(\tilde{Y}_{l-1}(b))\rangle}{\|\phi(Y_{l-1}(a))\|\|\phi(Y_{l-1}(b))\|}\right)^{2}\\
&+2\left(\E_{l}\frac{\langle\phi(\tilde{Y}_{l-1}(a)),\phi(\tilde{Y}_{l-1}(b))\rangle}{\|\phi(Y_{l-1}(a))\|\|\phi(Y_{l-1}(b))\|}-\E_{l}\frac{\langle\phi(\tilde{Y}_{l-1}(a)),\phi(\tilde{Y}_{l-1}(b))\rangle}{\|\phi(\tilde{Y}_{l-1}(a))\|\|\phi(\tilde{Y}_{l-1}(b))\|}\right)^{2}
\end{aligned}
\end{equation}

Using \cref{lem:hl} with $\zeta=12^{-1/2}d^{-1/2}\min\{\|a\|,\|b\|\}$,
then with probability at least $1-C_{2}n^{-1}$, where $C_{2}$ is
a constant that depends only on $\|a\|,\|b\|,d,\|\alpha\|_{S}$, we
have that $\|\phi(Y_{l-1}(a))\|\geq\frac{\|a\|}{2\sqrt{3d}}n^{1/2}$
and $\|\phi(Y_{l-1}(b))\|\geq\frac{\|b\|}{2\sqrt{3d}}n^{1/2}$. Therefore,
\begin{center}
\begin{align*}
&\left(\E_{l}\frac{\langle\phi(Y_{l-1}(a)),\phi(Y_{l-1}(b))\rangle}{\|\phi(Y_{l-1}(a))\|\|\phi(Y_{l-1}(b))\|}-\E_{l}\frac{\langle\phi(\tilde{Y}_{l-1}(a)),\phi(\tilde{Y}_{l-1}(b))\rangle}{\|\phi(Y_{l-1}(a))\|\|\phi(Y_{l-1}(b))\|}\right)^{2}\\  &\leq\frac{144d^{2}}{\|a\|^{2}\|b\|^{2}}n^{-2}\E\left(\langle\phi(Y_{l-1}(a)),\phi(Y_{l-1}(b))\rangle-\langle\phi(\tilde{Y}_{l-1}(a)),\phi(\tilde{Y}_{l-1}(b))\rangle\right)^{2}\\
 & \leq\frac{288d^{2}}{\|a\|^{2}\|b\|^{2}}n^{-2}\big[\E\left(\langle\phi(Y_{l-1}(a)),\phi(Y_{l-1}(b))-\phi(\tilde{Y}_{l-1}(b))\rangle\right)^{2}\\
 &+\E\left(\langle\phi(Y_{l-1}(a))-\phi(\tilde{Y}_{l-1}(a)),\phi(\tilde{Y}_{l-1}(b))\rangle\right)^{2}\huge]\\
 & \leq C_{3}n^{-1},
\end{align*}
\par\end{center}

where $C_{3}$ depends only on $\|a\|,\|b\|,d,\|\alpha\|_{S}$ and
where we have used Jensen's inequality (first line), the Lipschitz
property of ReLU, and the bounds on $\E\|Y_{l-1}-\tilde{Y}_{l-1}\|^{2}$
from \cref{thm:Depth-Uniform-strong-convergence}. Using the same techniques
for the second term in the RHS of \cref{eq:T12}, we obtain a similar bound, and we finally get
\begin{center}
$T_{12}=\left(\E_{l}\frac{\langle\phi(Y_{l-1}(a)),\phi(Y_{l-1}(b))\rangle}{\|\phi(Y_{l-1}(a))\|\|\phi(Y_{l-1}(b))\|}-\E_{l}\frac{\langle\phi(\tilde{Y}_{l-1}(a)),\phi(\tilde{Y}_{l-1}(b))\rangle}{\|\phi(\tilde{Y}_{l-1}(a))\|\|\phi(\tilde{Y}_{l-1}(b))\|}\right)^{2}\leq C_{4}n^{-1}$
\par\end{center}

where $C_{4}$ depends only on $\|a\|,\|b\|,d,\|\alpha\|_{S}$ .

\item It remains to bound the third term $T_{13}$ to conclude. We have
that
\begin{align*}
T_{13}&=\left(\E_{l}\frac{\langle\phi(\tilde{Y}_{l-1}(a)),\phi(\tilde{Y}_{l-1}(b))\rangle}{\|\phi(\tilde{Y}_{l-1}(a))\|\|\phi(\tilde{Y}_{l-1}(b))\|}-f(c_{l-1}(a,b))\right)^{2}\\
&\leq3\underbrace{\left(\E_{l}\frac{n^{-1}\langle\phi(\tilde{Y}_{l-1}(a)),\phi(\tilde{Y}_{l-1}(b))\rangle}{n^{-1/2}\|\phi(\tilde{Y}_{l-1}(a))\|n^{-1/2}\|\phi(\tilde{Y}_{l-1}(b))\|}-\E_{l}\frac{n^{-1}\langle\phi(\tilde{Y}_{l-1}(a)),\phi(\tilde{Y}_{l-1}(b))\rangle}{\sqrt{\frac{1}{2}q_{l-1}(a)}\sqrt{\frac{1}{2}q_{l-1}(b)}}\right)^{2}}_{T_{131}}\\
&+3\underbrace{\left(\E_{l}\frac{n^{-1}\langle\phi(\tilde{Y}_{l-1}(a)),\phi(\tilde{Y}_{l-1}(b))\rangle}{\sqrt{\frac{1}{2}q_{l-1}(a)}\sqrt{\frac{1}{2}q_{l-1}(b)}}-\E\frac{n^{-1}\langle\phi(\tilde{Y}_{l-1}(a)),\phi(\tilde{Y}_{l-1}(b))\rangle}{\sqrt{\frac{1}{2}q_{l-1}(a)}\sqrt{\frac{1}{2}q_{l-1}(b)}}\right)^{2}}_{T_{132}}\\
&+3\underbrace{\left(\E\frac{n^{-1}\langle\phi(\tilde{Y}_{l-1}(a)),\phi(\tilde{Y}_{l-1}(b))\rangle}{\sqrt{\frac{1}{2}q_{l-1}(a)}\sqrt{\frac{1}{2}q_{l-1}(b)}}-f(c_{l-1}(a,b))\right)^{2}}_{T_{133}}.
\end{align*}

Simple calculations yield 
\[
T_{131}=\left(\E\frac{n^{-1}\langle\phi(\tilde{Y}_{l-1}(a)),\phi(\tilde{Y}_{l-1}(b))\rangle}{n^{-1/2}\|\phi(\tilde{Y}_{l-1}(a))\|n^{-1/2}\|\phi(\tilde{Y}_{l-1}(b))\|}-\E\frac{n^{-1}\langle\phi(\tilde{Y}_{l-1}(a)),\phi(\tilde{Y}_{l-1}(b))\rangle}{\sqrt{\frac{1}{2}q_{l-1}(a)}\sqrt{\frac{1}{2}q_{l-1}(b)}}\right)^{2}\leq C_{5}n^{-1},
\]
where $C_{5}$ depends only on $\|a\|,\|b\|,d,\|\alpha\|_{S}$. 

For the second term $T_{132},$ we have that 
\begin{align*}
T_{132}&=\left(\E_{l}\frac{n^{-1}\langle\phi(\tilde{Y}_{l-1}(a)),\phi(\tilde{Y}_{l-1}(b))\rangle}{\sqrt{\frac{1}{2}q_{l-1}(a)}\sqrt{\frac{1}{2}q_{l-1}(b)}}-\E\frac{n^{-1}\langle\phi(\tilde{Y}_{l-1}(a)),\phi(\tilde{Y}_{l-1}(b))\rangle}{\sqrt{\frac{1}{2}q_{l-1}(a)}\sqrt{\frac{1}{2}q_{l-1}(b)}}\right)^{2}\\
&\leq(1-\mathbb{P}(\mathcal{H}_{a}^{l}\cap\mathcal{H}_{b}^{l}))^{2}\left(\E_{l}\frac{n^{-1}\langle\phi(\tilde{Y}_{l-1}(a)),\phi(\tilde{Y}_{l-1}(b))\rangle}{\sqrt{\frac{1}{2}q_{l-1}(a)}\sqrt{\frac{1}{2}q_{l-1}(b)}}-\E_{l}^{c}\frac{n^{-1}\langle\phi(\tilde{Y}_{l-1}(a)),\phi(\tilde{Y}_{l-1}(b))\rangle}{\sqrt{\frac{1}{2}q_{l-1}(a)}\sqrt{\frac{1}{2}q_{l-1}(b)}}\right)^{2}\\
&\leq C_{6}n^{-2},
\end{align*}

where we write $\E_{l}^{c}[.]=\E[.\mid(\mathcal{H}_{a}^{l}\cap\mathcal{H}_{b}^{l})^{c}]$,
and where $C_{6}$ is a constant that depends only on $\|a\|,\|b\|,d,\|\alpha\|_{S}.$

Now let us deal with the last term $T_{133}.$ Observe that $\E\frac{n^{-1}\langle\phi(\tilde{Y}_{l-1}(a)),\phi(\tilde{Y}_{l-1}(b))\rangle}{\sqrt{\frac{1}{2}q_{l-1}(a)}\sqrt{\frac{1}{2}q_{l-1}(b)}}=f(\tilde{c}_{l-1}(a,b))$
where $\tilde{c}_{l-1}(a,b):=\frac{\E_{l}\tilde{Y}_{l-1}^{1}(a)\tilde{Y}_{l-1}^{1}(b)}{\sqrt{q_{l-1}(a)}\sqrt{q_{l-1}(b)}}$.
Using the Lipschitz property of $f$, we obtain
\begin{align*}
T_{133}&=\left(\E\frac{n^{-1}\langle\phi(\tilde{Y}_{l-1}(a)),\phi(\tilde{Y}_{l-1}(b))\rangle}{\sqrt{\frac{1}{2}q_{l-1}(a)}\sqrt{\frac{1}{2}q_{l-1}(b)}}-f(c_{l-1}(a,b))\right)^{2} \\ &\leq\left(\tilde{c}_{l-1}(a,b)-c_{l-1}(a,b)\right)^{2}\\
 & =(q_{l-1}(a)q_{l-1}(b))^{-1}\left(\E\tilde{Y}_{l-1}^{1}(a)\tilde{Y}_{l-1}^{1}(b)-\tilde{q}_{l-1,\infty}(a,b)\right)^{2}\\
 & \leq2(q_{l-1}(a)q_{l-1}(b))^{-1}\left(\E(\tilde{Y}_{l-1}^{1}(a)\tilde{Y}_{l-1}^{1}(b)-\tilde{q}_{l-1,n}(a,b))^{2}+\E(\tilde{q}_{l-1,n}(a,b)-\tilde{q}_{l-1,\infty}(a,b))^{2}\right)\\
 & \leq C_{7}(n^{-1}+\tilde{\Delta}_{l-1,n})
\end{align*}

where we have used the bounds on $q_{l-1}$, and where $C_{7}$ is
a constant that depends only on $\|a\|,\|b\|,d,\|\alpha\|_{S}$. As
a result, we have that 
\[
T_{13}=\left(\E\frac{\langle\phi(\tilde{Y}_{l-1}(a)),\phi(\tilde{Y}_{l-1}(b))\rangle}{\|\phi(\tilde{Y}_{l-1}(a))\|\|\phi(\tilde{Y}_{l-1}(b))\|}-f(c_{l-1}(a,b))\right)^{2}\leq C_{8}(n^{-1}+\tilde{\Delta}_{l-1,n}),
\]
where $C_{8}$ is a constant that depends only on $\|a\|,\|b\|,d,\|\alpha\|_{S}$.

Combining all the results we obtain the following upperbound on $\tilde{\Delta}_{l,n}$
\begin{center}
\[
\tilde{\Delta}_{l,n}\leq(1+C_{9}\alpha_{l,L}^{4})\tilde{\Delta}_{l-1,n}+C_{10}\alpha_{l,L}^{2}n^{-1}
\]
\par\end{center}

where $C_{9},C_{10}$ are constants that depend only on $\|a\|,\|b\|,d,\|\alpha\|_{S}$.
Using the fact that $\tilde{\Delta}_{0,n}=0$, and that $\sum_{l=1}^{L}\alpha_{l,L}^{4}\leq\|\alpha\|_{S}^{2}$,
we obtain that

\[
\sup_{L\geq1}\sup_{l\in[L]}\tilde{\Delta}_{l,n}\leq C_{11}n^{-1},
\]

where $C_{11}$ depends only on $\|a\|,\|b\|,d,\|\alpha\|_{S}$.

\end{itemize}
We state this result formally in the next theorem.
\begin{thm}
\label{thm:cov_inf_width_tilde}Let $\alpha$ be a stable sequence
of scaling factors. There exists a constant $C>0$ that depends only
on $\|a\|,\|b\|,d,\|\alpha\|_{S}$, such that $\sup_{L\geq1}\sup_{l\in[L]}\E\left|\tilde{q}_{l,n}(a,b)-\tilde{q}_{l,\infty}(a,b)\right|^{2}\leq Cn^{-1}.$
\end{thm}
Combining the results of \cref{thm:cov_inf_width_tilde} and \cref{thm:Depth-Uniform-strong-convergence}, we obtain the following upperbound on the the difference between the neural covariance and the infinite-width covariance kernel. 
\begin{thm}
\label{thm:cov_inf_width}Let $\alpha$ be a stable sequence of scaling
factors. There exists a constant $C>0$ that depends only on $\|a\|,\|b\|,d,\|\alpha\|_{S}$,
such that $\sup_{L\geq1}\sup_{l\in[L]}\E\left|q_{l,n}(a,b)-\tilde{q}_{l,\infty}(a,b)\right|^{2}\leq Cn^{-1}.$
\end{thm}

We will see in the next section that the result of \cref{thm:cov_inf_width} is the cornerstone of commutativity; indeed, it suffices to study the infinite-depth limit of $\Tilde{q}_{l,\infty}$ to obtain commutativity with explicit convergence rates for width and depth. 

\section{Infinite-Depth Limits}\label{app:infinite_depth_proofs}

In this section, we study the infinite depth limit of the covariance
kernel for different choices of the sequence $\alpha$. We start by proving a general commutativity result.
\subsection{Infinite Depth Convergence of the Neural Covariance}

Now that we have depth-uniform bounds for the kernels, we can look
at what happens to the infinite-width kernels when depth increases.
\begin{thm}[General Commutativity]\label{thm:general_commutativity}
Let $\alpha \in \mathcal{S}$. Assume that the kernel $\tilde{q}_{\lfloor tL\rfloor,\infty}(a,b)$
converges to some limiting kernel $q_{t}^{\infty}(a,b)$ in the limit
$L\to\infty$ with some rate $r_{L}$. Then we have that 

\[
\sup_{t\in[0,1]}\|q_{\lfloor tL\rfloor,n}(a,b)-q_{t}^{\infty}(a,b)\|_{L_{2}}\le C\left(n^{-1/2}+r_{L}\right),
\]
 where $C$ is a constant that depends only on $\|a\|.\|b\|,d,\|\alpha\|_{S}$.
\end{thm}
\begin{proof}
We have that 
\begin{align*}
\|q_{\lfloor tL\rfloor,n}(a,b)-q_{t}^{\infty}(a,b)\|_{L_{2}} & \leq\|q_{\lfloor tL\rfloor,n}(a,b)-\tilde{q}_{\lfloor tL\rfloor,\infty}(a,b)\|_{L_{2}}\\
 & +\|\tilde{q}_{\lfloor tL\rfloor,\infty}(a,b)-q_{t}^{\infty}(a,b)\|_{L_{2}}\\
 & \leq C_{1}n^{-1/2}+r_{L},
\end{align*}

where we have used
\cref{thm:cov_inf_width} to obtain the constant
$C_{1}$ which depends only on $\|a\|.\|b\|,d,\|\alpha\|_{S}$.
We conclude the proof by taking $C = \max(C_1,1)$.\\
\end{proof}

The result of \cref{thm:general_commutativity} requires that the kernel $\tilde{q}_{\lfloor tL\rfloor,\infty}(a,b)$ converges in the infinite-depth limit with some rate $r_L$. In the next (sub)sections, we refine our analysis and study two scenarios where such convergence occurs.
\subsection{Sequence of Scaling factors as `Quasi-Convergent' Series. }
Recall that for $L \geq 1, l \in [L]$
$$\tilde{q}_{l,\infty}(a,b)=\tilde{q}_{l-1,\infty}(a,b)+\alpha_{l,L}^{2}(1/2q_{l-1}(a))^{1/2}(1/2q_{l-1}(b))^{1/2}f(c_{l-1}(a,b)),$$
where $c_{l-1,\infty}(a,b):=\frac{\tilde{q}_{l-1,\infty}(a,b)}{\tilde{q}_{l-1,\infty}(a,a)^{1/2}\tilde{q}_{l-1,\infty}(b,b)^{1/2}}$
is the infinite-width correlation kernel.\\

Given a sequence of scaling factors $\alpha$, define $Q_{l}^{\alpha}(a,b)=\tilde{q}_{l,\infty}(a,b)$
with the scaling factors being $\alpha_{l,L}$. 

To analyze the infinite-depth behaviour of the kernel $Q^\alpha_l$, it is crucial to understand the sensitivity of $Q_{l}^{\alpha}$ to $\alpha$. The first result characterizes the sensitivity of the variance to a change in $\alpha$.
\begin{lemma}
\label{lem:variance_terms_depth}Consider two stable sequences of
scaling factors $\alpha,\beta \in \sss$. Then, for all $L\geq1,l\in\{1,\dots,L\}$,
we have that
\[
\sup_{l\in\{1,\dots,L\}}|Q_{l}^{\alpha}(a,a)-Q_{l}^{\beta}(a,a)|\leq\frac{\|a\|^{2}}{2d}e^{\sup\{\|\alpha\|_{S}^{2},\|\beta\|_{S}^{2}\}}\sum_{l=1}^{L}|\alpha_{l,L}^{2}-\beta_{l,L}^{2}|
\]
\end{lemma}
\begin{proof}
To alleviate the notation, we write $Q_{l}^{\alpha}=Q_{l}^{\alpha}(a,a)$.
We have that 

\begin{center}
\begin{align*}
|Q_{l}^{\alpha}-Q_{l}^{\beta}| & \leq|Q_{l-1}^{\alpha}-Q_{l-1}^{\beta}|+\frac{1}{2}Q_{l-1}^{\alpha}|\alpha_{l,L}^{2}-\beta_{l,L}^{2}|+\frac{1}{2}\beta_{l,L}^{2}|Q_{l-1}^{\alpha}-Q_{l-1}^{\beta}|\\
 & \leq(1+\frac{1}{2}\beta_{l,L}^{2})|Q_{l-1}^{\alpha}-Q_{l-1}^{\beta}|+\frac{\|a\|^{2}}{2d}e^{\|\alpha\|_{S}^{2}/2}|\alpha_{l,L}^{2}-\beta_{l,L}^{2}|\\
\end{align*}
\par\end{center}

The result follows immediately by induction.\\
\end{proof}

Next, we prove a similar result for the kernel (not just the variance terms).

\begin{lemma}\label{lemma:diff_Q}
Consider two sequences of scaling factors $\alpha, \beta \in \sss$. Then, for all $L\geq1,l\in[L]$, we have that

\[
\sup_{l\in\{1,\dots,L\}}|Q_{l}^{\alpha}(a,b)-Q_{l}^{\beta}(a,b)|\leq C\sum_{l=1}^{L}|\alpha_{l,L}^{2}-\beta_{l,L}^{2}|,
\]
where $C$ is a constant that depends only on $\|a\|,\|b\|,d,\|\alpha\|_{S},\|\beta\|_{S}$.
\end{lemma}

\begin{proof}
Let $I_{l}(\alpha)=(1/2Q_{l-1}^{\alpha}(a,a))^{1/2}(1/2Q_{l-1}^{\alpha}(b,b))^{1/2}$.
We also use the notation $c_{l}^{\alpha}$ to denote the previously
defined correlation kernel $c_{l}$. We have that
\begin{equation}\label{eq:diff_Q}
\begin{aligned}
|Q_{l}^{\alpha}(a,b)-Q_{l}^{\beta}(a,b)| & \leq|Q_{l-1}^{\alpha}(a,b)-Q_{l-1}^{\beta}(a,b)|+|\alpha_{l,L}^{2}-\beta_{l,L}^{2}|I_{l}(\alpha)\\
 & +\beta_{l,L}^{2}|I_{l}(\alpha)f(c_{l-1}^{\alpha}(a,b))-I_{l}(\beta)f(c_{l-1}^{\beta}(a,b))|
\end{aligned}
\end{equation}

Using \cref{lem:hl}, we have that $I_{l}(\alpha)\leq\frac{\|a\|\|b\|}{2d}e^{\|\alpha\|_{S}^{2}/2}$ and $I_{l}(\beta)\leq\frac{\|a\|\|b\|}{2d}e^{\|\beta\|_{S}^{2}/2}$. Moreover, using \cref{lem:variance_terms_depth}, we have that 

\begin{align*}
|I_{l}(\alpha)-I_{l}(\beta)| & \leq(1/2Q_{l-1}^{\alpha}(a,a))^{1/2}|(1/2Q_{l-1}^{\alpha}(b,b))^{1/2}-1/2Q_{l-1}^{\beta}(b,b))^{1/2}|\\
 & +(1/2Q_{l-1}^{\beta}(b,b))^{1/2}|(1/2Q_{l-1}^{\alpha}(a,a))^{1/2}-1/2Q_{l-1}^{\beta}(a,a))^{1/2}|\\
 & \leq C_{1}\sum_{l=1}^{L}|\alpha_{l,L}^{2}-\beta_{l,L}^{2}|,
\end{align*}
where $C_{1}$ is a constant that depends only on $\|a\|,\|b\|,d,\|\alpha\|_{S}$,
$\|\beta\|_{S}$, and where we have used the fact that $|\sqrt{x_{1}}-\sqrt{x_{2}}|\leq\frac{1}{2\sqrt{x_{0}}}|x_{1}-x_{2}|$
for $x_{1},x_{2}>x_{0}>0$.

For the third term in the RHS of \cref{eq:diff_Q}, we have that 

\begin{align*}
|I_{l}(\alpha)f(c_{l-1}^{\alpha}(a,b))-I_{l}(\beta)f(c_{l-1}^{\beta}(a,b))| & \leq|I_{l}(\alpha)-I_{l}(\beta)|+|f(c_{l-1}^{\alpha}(a,b))-f(c_{l-1}^{\beta}(a,b))|I_{l}(\beta)\\
 & \leq C_{1}\sum_{l=1}^{L}|\alpha_{l,L}^{2}-\beta_{l,L}^{2}|+C_{2}|Q_{l-1}^{\alpha}(a,b)-Q_{l-1}^{\beta}(a,b)|,
\end{align*}

where $C_{2}$ is a constant that depends only on $\|a\|,\|b\|,d,\|\alpha\|_{S}$,
$\|\beta\|_{S}$,. As a result, we obtain 

\begin{align*}
|Q_{l}^{\alpha}(a,b)-Q_{l}^{\beta}(a,b)|&\leq(1+C_{3}\beta_{l,L}^{2})|Q_{l-1}^{\alpha}(a,b)-Q_{l-1}^{\beta}(a,b)|\\
&+C_{4}|\alpha_{l,L}^{2}-\beta_{l,L}^{2}|+C_{5}\beta_{l,L}^{2}\sum_{l=1}^{L}|\alpha_{l,L}^{2}-\beta_{l,L}^{2}|,
\end{align*}
where $C_{3},C_{4},C_{5}$ are constants that depends only on $\|a\|,\|b\|,d,\|\alpha\|_{S}$,
$\|\beta\|_{S}$,. 

Using the fact that $Q_{0}^{\alpha}=Q_{0}^{\beta}$, we obtain that
there exists a constant $C$ that depends only on $\|a\|,\|b\|,d,\|\alpha\|_{S}$,
$\|\beta\|_{S}$, such that 
\[
|Q_{l}^{\alpha}(a,b)-Q_{l}^{\beta}(a,b)|\leq C\sum_{l=1}^{L}|\alpha_{l,L}^{2}-\beta_{l,L}^{2}|.
\]

\end{proof}

Let us begin by proving convergence in the case where $\alpha$ is the truncation (at level $L$) of a convergent series. The convergence is straightforward in this case and similar results have appeared in \citep{hayou21stable}.
\begin{lemma}\label{lemma:infinite_depth_conv_series}
Let $\alpha \in \mathcal{S}$  such that there
exists a sequence $\zeta=(\zeta_{i})_{i\geq1}\in\ell_{2}(\mathbb{N})$
such that $\alpha_{l,L}=\zeta_{l}$ for all $L\geq1,l\in[L]$. Then,
there exists a limiting kernel $Q_{\infty}^{\alpha}$ such that 
$$|Q_{L}^{\alpha}(a,b)-Q_{\infty}^{\alpha}(a,b)|\leq C\sum_{l\geq L}\zeta_{l}^{2},$$
where $C$ is a constant that depends only on $\|a\|,\|b\|,d,\|\alpha\|_{S}=\sqrt{\sum_{l=1}^{\infty}\zeta_{l}^{2}}$.
\end{lemma}
\begin{proof}
Let $L'\geq L\geq1$. It is straightforward that there exists a constant
$C_{1}>0$ that depends only on $\|a\|,\|b\|,d,\|\alpha\|_{S}$ such
that 
\[
|Q_{L}^{\alpha}(a,b)-Q_{L'}^{\alpha}(a,b)|\leq C_{1}\sum_{L<l\leq L'}\zeta_{l}^{2},
\]

which shows that $(Q_{L}^{\alpha})_{L\geq1}$ is a Cauchy sequence
and therefore it converges to some limit $Q_{\infty}^{\alpha}$. Taking
$L'$ to infinity provide the convergence rate.\\

\end{proof}
Note that we only show the existence of the limit (and the convergence rate) in \cref{lemma:infinite_depth_conv_series}. Some properties of the limiting kernel in this case were studied in \cite{hayou21stable}, these include continuity, universality etc.\\

Combining the results of \cref{lemma:infinite_depth_conv_series} and \cref{lemma:diff_Q}, we conclude the following.
\begin{lemma}\label{lemma:infinite_depth_quasi_convergent}
Let $\alpha \in \sss$. Assume that there
exists a sequence $\zeta=(\zeta_{i})_{i\geq1}\in\ell_{2}(\mathbb{N})$
such that $\sum_{l=1}^{L}|\alpha_{l,L}^{2}-\zeta_{l}^{2}|\to0$ as
$L\to\infty$. Then, we have 
\[
\sup_{l\in\{1,\dots,L\}}|Q_{l}^{\alpha}(a,b)-Q_{l}^{\zeta}(a,b)|\to0,
\]
where the convergence rate is given by $r_{L}=\Theta(\sum_{l=1}^{L}|\alpha_{l,L}^{2}-\zeta_{l}^{2}|+\sum_{l\geq L}\zeta_{l}^{2})$.
\end{lemma}
Combining \cref{thm:general_commutativity} and \cref{lemma:infinite_depth_quasi_convergent}, we conclude for \cref{thm:main_ConvSeries}.

\subsection{Convergence with ``Normalized'' Sequences}

For the specific choice of $\alpha_{l,L}=L^{-1/2},$ we know from
\citep{Hayou2023WidthDepth} that the covariance kernel $\tilde{q}_{\lfloor tL\rfloor,\infty}(a,b)$
converges to the solution of the following ODE

\begin{equation}\label{eq:ode}
\frac{dq_{t}(a,b)}{dt}=\frac{e^{t/2}}{2}\zeta(a,b)f(\zeta(a,b)^{-1}e^{-t/2}q_{t}(a,b))=F(t,q_{t}(a,b)),
\end{equation}
where $\zeta(a,b) = d^{-1} \|a\| \|b\|$.
The Euler scheme of \cref{eq:ode} is given by 
\[
q_{l}^{E}(a,b)=q_{l-1}^{E}(a,b)+\alpha_{l,L}^{2}F(t_{l-1},q_{l-1}^{E}(a,b)), \quad q_0^{E}(a,b) = q_0(a,b),
\]
where $t_{l}=\sum_{i=1}^{l}\alpha_{l,L}^{2}$.\\

For an ODE of the form $\dot{z}(t)=F(t, z(t))$, we call $F$ the ODE functional. It is well known that under some conditions on this functional, the discretization error of the Euler scheme with steps $\delta_1, \dots, \delta_L$ is of order $\mathcal{O}(\max_{i \in [L]}\delta_i)$.

\begin{thm}[Corollary of Thm 212A in \citep{butcher2003}]\label{thm:butcher}
Consider an ODE of the form $\dot{z}(t)=F(t, z(t)), t\in [0,1]$, and consider the Euler discretization scheme with steps $\delta_1, \dots, \delta_L$ given by $z^E_l = z^E_{l-1} + \delta_l F(t_{l-1}, z^E_{l-1})$ with the initial condition $z^E_0 = z_0$, where $t_{l} = \sum_{i=1}^l \delta_i$. Assume that the following hold
\begin{itemize}
    \item There exists a constant $L > 0$ such that $|F(t,z)-F(t,z')| \leq |z-z'|$ for all $t\in [0,1], z,z' \in \reals$.
    \item $M=\sup_{t\in[0,1]}\left|\frac{d^2 z(t)}{dt^2}\right| < \infty$.
\end{itemize}
Then, we have 
$$
\sup_{l\in[L]} |z^E_l - z_{t_l}| \leq C \max_{i \in [L]} \delta_i,
$$
where $C$ depends only on $M$ and $L$.
\end{thm}

In the next result, we use this result to show convergence rate of the Euler scheme presented above.

\begin{lemma}\label{lemma:diff_euler_flow}
Consider a normalized sequence of scaling fatcors $\alpha$ and let
$h_{L}=\max_{1\leq l\leq L}\alpha_{l,L}^{2}$. We have
$$\sup_{1\leq l\leq L}|q_{l}^{E}(a,b)-q_{t_{l}}(a,b)|\leq C\,h_{L},$$
where $C$ is a constant that depends only on $\|a\|,\|b\|,d$.
\end{lemma}
\begin{proof}
Let us verify the conditions of \cref{thm:butcher} one by one.
\begin{itemize}
    \item Lipschitz property: from \cref{lemma:lipschitz_f}, we know that the function $f$ is Lipschitz. Therefore, we have $|F(t,z) - F(t,z')| \leq \frac{1}{2}e^{t/2}|\zeta(a,b)| |e^{-t/2}\zeta(a,b)^{-1}(z-z')| = |z-z'|$.
    \item For $t \in [0,1]$, we have 
    \begin{align*}
        \frac{d^2 q_t(a,b)}{dt^2} = &\frac{t e^{t/2}}{4} \zeta(a,b) f(\zeta(a,b)^{-1} e^{-t/2} q_t(a,b)) + \frac{e^{t/2}}{2}\zeta(a,b)(-\zeta(a,b)^{-1} \frac{t e^{t/2}}{2} q_t(a,b) \\
        &+ \zeta(a,b)^{-1} e^{-t/2} \frac{d q_t(a,b)}{dt}) f'(\zeta(a,b)^{-1} e^{-t/2} q_t(a,b)),
    \end{align*}
    Replacing $\frac{d q_t(a,b)}{dt}$ by its valye, it is straightforward that $M = \sup_{t \in [0,1]} \left|\frac{d^2 q_t(a,b)}{dt^2}\right|$ is finite and depends only on $\|a\|,\|b\|, d$.
\end{itemize}
This concludes the proof.\\
\end{proof}

Now it remains to bound the difference between $q_{l}^{E}(a,b)$ and
$\tilde{q}_{l,\infty}(a,b)$, the covariance kernel of the auxiliary process. We deal with this in the next result.

\begin{lemma}\label{lemma:diff_tilde_euler}
Consider a normalized sequence of scaling factors $\alpha$.
Let $h_{L}=\max_{1\leq l\leq L}\alpha_{l,L}^{2}$ and assume that
$Lh_{L}^{2}=o(1).$ Then, we have 
\[
\sup_{1\le l\leq L}|\tilde{q}_{l,\infty}(a,b)-q_{l}^{E}(a,b)|\leq CLh_{L}^{2},
\]
where $C$ is a constant that depends only on $\|a\|,\|b\|,d$.
\end{lemma}
\begin{proof}
Assume that $L\times h_{L}^{2}=o(1).$ We write $\zeta = \zeta(a,b)$ to simplify the notation. We have 
\begin{equation}\label{eq:diff_tilde_euler}
\begin{aligned}
|\tilde{q}_{l,\infty}(a,b)&-q_{l}^{E}(a,b)|  \leq|\tilde{q}_{l-1,\infty}(a,b)-q_{l-1}^{E}(a,b)|\\
&+\frac{1}{2}\alpha_{l,L}^{2}\zeta f(\zeta^{-1}\prod_{i=1}^{l-1}(1+\frac{\alpha_{l,L}^{2}}{2})^{-1}\tilde{q}_{l-1,\infty}(a,b))|\prod_{i=1}^{l-1}(1+\frac{\alpha_{l,L}^{2}}{2})-e^{\frac{1}{2}\sum_{i=1}^{l-1}\alpha_{l,L}^{2}}|\\
 & +\frac{1}{2}\alpha_{l,L}^{2}\zeta e^{\frac{1}{2}\sum_{i=1}^{l-1}\alpha_{l,L}^{2}}|f(\zeta^{-1}\prod_{i=1}^{l-1}(1+\frac{\alpha_{l,L}^{2}}{2})^{-1}\tilde{q}_{l-1,\infty}(a,b))-f(\zeta^{-1}e^{-\frac{1}{2}\sum_{i=1}^{l-1}\alpha_{l,L}^{2}}q_{l-1}^{E}(a,b))|
\end{aligned}
\end{equation}

Notice that
\begin{align*}
\prod_{i=1}^{l-1}(1+\frac{\alpha_{l,L}^{2}}{2}) & =\prod_{i=1}^{l-1}(e^{\frac{1}{2}\alpha_{l,L}^{2}}+\mathcal{O}(\alpha_{l,L}^{4}))=\prod_{i=1}^{l-1}e^{\frac{1}{2}\alpha_{l,L}^{2}}(1+\mathcal{O}(\alpha_{l,L}^{4}))\\
 & =e^{\frac{1}{2}\sum_{i=1}^{l-1}\alpha_{l,L}^{2}}(1+\mathcal{O}(h_{L}^{2}))^{l-1}=e^{\frac{1}{2}\sum_{i=1}^{l-1}\alpha_{l,L}^{2}}+\mathcal{O}(Lh_{L}^{2}),
\end{align*}
where the constant in ``$\mathcal{O}$'' is universal. As a result,
there exists a constant $C_{1}$ that depends only on $\|a\|,\|b\|,d,$
such that the second term in the RHS of \cref{eq:diff_tilde_euler} is smaller than $ C\times Lh_{L}^{2}.$ 

We also have
\[
\prod_{i=1}^{l-1}(1+\frac{\alpha_{l,L}^{2}}{2})^{-1}=\prod_{i=1}^{l-1}e^{\frac{1}{2}\alpha_{l,L}^{2}}+\mathcal{O}(Lh_{L}^{2}),
\]
where the constant in ``$\mathcal{O}$'' is universal. Using the Lipschitz property of $f$ (\cref{lemma:lipschitz_f}), we obtain that

\begin{align*}
|f(\zeta^{-1}\prod_{i=1}^{l-1}(1+\frac{\alpha_{l,L}^{2}}{2})^{-1}\tilde{q}_{l-1,\infty}(a,b))&-f(\zeta^{-1}e^{-\frac{1}{2}\sum_{i=1}^{l-1}\alpha_{l,L}^{2}}q_{l-1}^{E}(a,b))| \\
& \leq\zeta^{-1}|\prod_{i=1}^{l-1}(1+\frac{\alpha_{l,L}^{2}}{2})^{-1}-e^{-\frac{1}{2}\sum_{i=1}^{l-1}\alpha_{l,L}^{2}}|\tilde{q}_{l-1,\infty}(a,b)\\
 & +\zeta^{-1}e^{-\frac{1}{2}\sum_{i=1}^{l-1}\alpha_{l,L}^{2}}|\tilde{q}_{l-1,\infty}(a,b)-q_{l-1}^{E}(a,b)|,
\end{align*}
Therefore, 
\[
|\tilde{q}_{l,\infty}(a,b)-q_{l}^{E}(a,b)|\leq(1+C_{2}\alpha_{l,L}^{2})|\tilde{q}_{l-1,\infty}(a,b)-q_{l-1}^{E}(a,b)|+C_{3}Lh_{L}^{2},
\]
where $C_{2},C_{3}$ are constants that depend only on $\|a\|,\|b\|,d$. An induction argument allows us to conclude.\\

\end{proof}

Combining the results of \cref{lemma:diff_euler_flow} and \cref{lemma:diff_tilde_euler}, we obtain the following result.
\begin{thm}\label{thm:infinite_depth_normalized_kernel}
Consider a sequence of scaling factors $\alpha$ such that $\sum_{l=1}^{L}\alpha_{l,L}^{2}=1$.
Let $h_{L}=\max_{1\leq l\leq L}\alpha_{l,L}^{2}$ and assume that
$Lh_{L}^{2}=o(1).$ Then, we have that

\[
\sup_{1\le l\leq L}|\tilde{q}_{l,\infty}(a,b)-q_{t_{l}}(a,b)|\leq C(h_{L}+Lh_{L}^{2})
\]
\end{thm}

By combining the results of \cref{thm:general_commutativity} and \cref{thm:infinite_depth_normalized_kernel}, we obtain the first part of \cref{thm:main_Normalized}. It remains to show the second part of the theorem when $\sup_{t \in [0,1]}\left|\sum_{k=1}^{\lfloor t L\rfloor} \alpha_{k,L}^2-\lambda(t)\right| \leq r_L$ and $\lim_{L \to\infty} r_L = 0$. Assume that this condition holds. From the ODE \cref{eq:ode}, it is straightforward that $|q_t(a,b) - q_{t'}(a,b)| \leq C_1 |t - t'|$ holds for all $t,t' \in [0,1]$ for some constant $C_1>0$ that depends only on $\zeta(a,b)$. \\
Let $t\in [0,1]$ and $t_L = \sum_{k=1}^{\lfloor t L\rfloor} \alpha_{k,L}^2$. As a result of the inequality above, we have $|q_{t_L}(a,b) - q_{\lambda(t)}| \leq C_1 |t_L -\lambda(t)| \leq C_1 r_L$. We conclude using the first part of the theorem and the triangular inequality.

\section{Other Technical Results}

\subsection{Lemma for the Auxilliary process}
We use the next lemma to prove that the Auxilliary process has \iid coordinates. This is a trivial result, but we include the proof.
\begin{lemma}\label{lemma:gaussian_vec}
Let $W \in \reals^{n\times n}$ be a matrix of standard Gaussian random variables $W_{ij} \sim \normal(0,1)$. Let $v \in \reals^n$ be a random vector independent from $W$ and satisfies $\|v\|_2 = 1$ . Then, $W v \sim \normal(0, I)$.
\end{lemma}
\begin{proof}
The proof follows a simple characteristic function argument. Indeed, by conditioning on $v$, we observe that $Wv \sim \normal(0, I)$. Let $u \in \reals^n$, we have that \begin{align*}
    \E_{W,v}[e^{i \langle u, Wv\rangle}]  &=  \E_v[ \E_W[e^{i \langle u, Wv\rangle}| v]] \\
    &= \E_v[ e^{-\frac{\|u\|^2}{2}}] \\
    &= e^{-\frac{\|u\|^2}{2}}.\\
\end{align*}
This concludes the proof as the latter is the characteristic function of a random Gaussian vector with Identity covariance matrix.
\end{proof}

\subsection{Lemma for the (correlation) function $f$}

\begin{lemma}[Function $f$]\label{lemma:lipschitz_f}
Let $f:[-1,1] \to [-1,1]$ be the function defined by $f(c):=2\E[\phi(Z_{1})\phi(cZ_{1}+\sqrt{1-c^{2}}Z_{2})]$. Then, we have 
$$
f(c) = \frac{1}{\pi}(c \arcsin{c} + \sqrt{1-c^2})+\frac{1}{2}c.
$$
As a result, $f$ is $1$-Lipschitz.
\end{lemma}
\begin{proof}
The closed-form expression of $f$ has appeared in a series of papers under different forms \citep{cho2009, hayou2019impact, hayou21stable}. Here, we only show the Lipschitz property which is straightforward. From the closed-form expression of $f$, we obtain
$$
f'(c) = \pi^{-1} \arcsin(c) + \frac{1}{2},
$$
which shows that $|f'|\leq 1$ and concludes the proof.
\end{proof}
\end{document}